\documentclass[accepted]{uai2025_modified}
\usepackage[american]{babel}
\usepackage{algorithm}
\usepackage{algpseudocode}
\usepackage{amsfonts}
\usepackage{amsmath}
\usepackage{amsthm}
\usepackage{authblk}
\usepackage{bbm}
\usepackage{cancel}
\usepackage{chapterbib}
\usepackage{cleveref}
\usepackage{graphicx}
\usepackage{mathrsfs}
\usepackage{multirow}
\usepackage[round,sectionbib]{natbib}

\usepackage{soul}
\usepackage{subcaption}
\usepackage{tikz}
\usetikzlibrary{positioning}
\usepackage{xcolor}
\usepackage{url}
\hypersetup{colorlinks = True, linkcolor = violet, citecolor = blue}

\theoremstyle{plain}
\newtheorem{theorem}{Theorem}[section]
\newtheorem{proposition}[theorem]{Proposition}
\newtheorem{lemma}[theorem]{Lemma}

\newtheorem{definition}[theorem]{Definition}
\newtheorem{assumption}[theorem]{Assumption}
\Crefname{assumption}{Assumption}{Assumptions}
\newtheorem{remark}[theorem]{Remark}

\newcommand{\E}[2]{\mathop{\mathbb{E}}_{#1} \left[ #2 \right]}
\newcommand{\crossent}[2]{\mathrm{H}\left( #1 ~ \vert \vert ~ #2 \right)}
\newcommand{\ent}[1]{\mathrm{H}\left( #1 \right)}
\newcommand{\kld}[2]{\mathrm{D_{KL}} \left( #1 ~ \vert \vert ~ #2 \right)}
\newcommand{\informalAssumption}[2]{%
    \vspace{2mm}
    \noindent \textbf{#1} (informal)\textbf{.}
    \textit{#2}
    \vspace{2mm}
}

\newcommand{\algAbbrev}{PROMPT}
\newcommand{\algName}{proxy-informed robust method for probabilistic transfer learning}
\newcommand{\AlgName}{Proxy-informed RObust Method for Probabilistic Transfer learning}
\newcommand{\auxInfo}{\mathbf{z}}
\newcommand{\auxInfoRV}{\mathbf{Z}}
\newcommand{\auxvalues}{\mathscr{Z}}
\newcommand{\cDelta}{\Delta^c}
\newcommand{\cIG}{\mathrm{IG}^c}
\newcommand{\csim}[2]{\mathcal{R}_{#1} ( #2 )}
\newcommand{\datavalues}{\mathscr{D}}
\newcommand{\param}{\boldsymbol{\theta}}
\newcommand{\paramRV}{\boldsymbol{\Theta}}
\newcommand{\paramvalues}{\mathscr{T}}
\newcommand{\rDelta}{\Delta^{\mathcal{R}}}
\newcommand{\rIG}{\mathrm{IG}^{\mathcal{R}}}
\newcommand{\robustL}{L^{\mathcal{R}}}
\newcommand{\sourcedata}{\mathbf{d}}
\newcommand{\sourcedataRV}{\mathbf{D}}
\newcommand{\sourcex}{\mathbf{x}}
\newcommand{\sourcey}{\mathbf{y}}
\newcommand{\targetconcept}{\boldsymbol{\theta}^{\star}}
\newcommand{\tsparam}{\boldsymbol{\psi}}
\newcommand{\truetsparam}{\boldsymbol{\psi}^{\star}}
\newcommand{\tsparamvalues}{\mathscr{S}}
\newcommand{\targettsparam}{\boldsymbol{\psi}_{n+1}^{\star}}
\newcommand{\targettsparamRV}{\boldsymbol{\Psi_{n+1}}}
\newcommand{\tsparamRV}{\boldsymbol{\Psi}}
\newcommand{\targetdata}{\mathbf{d}_{n+1}}
\newcommand{\targetdataRV}{\mathbf{D}_{n+1}}
\newcommand{\truedataRV}{\mathbf{D}^{\star}}
\newcommand{\weight}{\boldsymbol{\eta}}

\newcommand{\colorgrey}[1]{{\color{lightgray} #1}}

\newcommand{\ourTitle}{Proxy-informed Bayesian transfer learning with unknown sources}
\title{\ourTitle}

\author[1]{\href{mailto:<sabina.sloman@manchester.ac.uk>?Subject=Your UAI 2025 paper}{Sabina~J.~Sloman}{}}
\author[2]{Julien~Martinelli\thanks{Work done while at Inserm Bordeaux Population Health, Vaccine Research Institute, Université de Bordeaux, Inria Bordeaux Sud-ouest, France.}}
\author[1,2,3]{Samuel~Kaski}
\affil[1]{%
    Department of Computer Science\\
    University of Manchester\\
    Manchester, UK
}
\affil[2]{%
   Department of Computer Science\\
   Aalto University\\
   Espoo, Finland
}
\affil[3]{%
    ELLIS Institute Finland\\
    Helsinki, Finland
  }
  
\begin{document}
\bibliographystyle{plainnat}

\maketitle

\begin{abstract}
    Generalization outside the scope of one's training data requires leveraging prior knowledge about the effects that transfer, and the effects that don't, between different data sources.
    Transfer learning is a framework for specifying and refining this knowledge about sets of source (training) and target (prediction) data.
    A challenging open problem is addressing the empirical phenomenon of negative transfer, whereby the transfer learner performs worse on the target data after taking the source data into account than before.
    We first introduce a Bayesian perspective on negative transfer, and then a method to address it.
    The key insight from our formulation is that negative transfer can stem from misspecified prior information about non-transferable causes of the \textit{source} data.
    Our proposed method, \algName{} (\algAbbrev{}), does not require prior knowledge of the source data (the data sources may be ``unknown'').
    \algAbbrev{} is thus applicable when differences between tasks are unobserved, such as in the presence of latent confounders.
    Moreover, the learner need not have access to observations in the target task (may not have the ability to ``fine-tune''), and instead makes use of proxy (indirect) information. 
    Our theoretical results show that the threat of negative transfer does not depend on the informativeness of the proxy information, highlighting the usefulness of \algAbbrev{} in cases where only noisy indirect information, such as human feedback, is available.
\end{abstract}

The paradigm of transfer learning takes, often sparse, data from a set of \textit{source} tasks and uses them to predict outcomes in a different but related \textit{target} task.
Consider the task of predicting the effectiveness of a treatment for a new patient on the basis of observational data.
Inevitably, the measured effects of the treatment in the source data are affected by a myriad of unobserved confounders, such as the quality of treatment in a given clinical setting or the patient's adherence to a treatment regimen.
Prediction in this setting requires learning both \textit{shared} parameters (the treatment effect) and \textit{task} parameters (the quality of treatment in \textit{this} patient's local clinic; \textit{this} patient's adherence).
Failure to account for the presence of task-specific effects can result in imprecise or inaccurate predictions in the target task.
Bayesian learning is a natural paradigm to apply in settings where data is sparse and prior information is available.
When this prior information is reliable, the Bayesian transfer learner can leverage the source data to make accurate and calibrated predictions when encountering new target tasks.

In practice, however, the Bayesian transfer learner often experiences \textit{negative transfer}, performing worse in the target task after taking the source data into account than before.
Understanding the conditions under which negative transfer occurs, and how to address it, is a challenging open problem \citep{suder_bayesian_2023}.

Our first contribution is to provide a precise and Bayesian characterization of the phenomenon of negative transfer.
Our formulation treats the Bayesian transfer learner's objective as a special case of inference in the presence of nuisance parameters, and applies results from this more general class of problems to elucidate the conditions under which negative transfer occurs.
We show that negative transfer can arise when prior information about source task parameters is unavailable or mistaken.
This result implies that alleviating the threat of negative transfer requires removing the learner's reliance on, possibly mistaken, prior information about source task parameters.
    
Our second contribution is to propose a method, \algName{} (\algAbbrev{}), that allows the learner to form a posterior predictive distribution in the target task without such prior information.
Our third contribution is to use our formulation of the Bayesian transfer learner's objective to provide theoretical guarantees on \algAbbrev{}'s ability to alleviate the threat of negative transfer.
    
\algAbbrev{} operates in a setting that differs from, and is in some ways more general than, settings in existing literature on probabilistic approaches to transfer and meta-learning \citep{grant_recasting_2018,yoon_bayesian_2018,gordon_meta-learning_2019,patacchiola_bayesian_2020}.
We here discuss some differences between our setting and probabilistic meta-learning, Bayesian meta-learning, and proxy methods for multi-source domain adaptation.
We discuss these and other related works in more detail in \Cref{sec:related-work}.

\noindent \textit{Difference \#1: \algAbbrev{} can cope with unknown sources.}
    Probabilistic meta-learning \citep{gordon_meta-learning_2019}, Bayesian meta-learning \citep{grant_recasting_2018,yoon_bayesian_2018,patacchiola_bayesian_2020}, and multi-domain adaptation \citep{tsai2024proxy} assume the availability of some prior knowledge about the source data, such as the number of distinct tasks represented in the source data and which data points correspond to the same task.
    This assumption would be violated in our motivating example: Each patient's outcomes are influenced by confounders whose values are unknown, and the learner cannot know which outcomes are influenced by the same latent confounder values.
    Probabilistic meta-learning \citep{gordon_meta-learning_2019} also requires that the target task arises from the same distribution as the source tasks. 
    \algAbbrev{} requires neither the availability of prior information about the source tasks nor that the target task resembles the source tasks.

\noindent \textit{Difference \#2: \algAbbrev{} relies on proxy information instead of fine-tuning.}
    Bayesian meta-learning requires that the learner has access to data from the target task in order to fine-tune their estimates \citep{grant_recasting_2018,yoon_bayesian_2018,patacchiola_bayesian_2020}.
    \algAbbrev{} relies on \textit{proxy} (indirect) information about the target task.
    Examples of proxy information are human feedback (e.g., to prompts such as ``What is the quality of treatment in the target hospital?'') and instrumental variables (e.g., hospital funding as an instrument for quality of care).
    This leads to connections between our work and the paradigm of proximal causal learning \citep{kuroki_measurement_2014,tchetgen_introduction_2020,alabdulmohsin_adapting_2023,tsai2024proxy}; see \Cref{sec:related-work} for a more detailed discussion.
    Unlike other proxy methods for multi-domain adaptation \citep{tsai2024proxy}, we distinguish between shared and task parameters and require additional techniques for estimation of the shared parameter.
    \algAbbrev{} does assume that the proxy information does not depend on the shared parameters.

Our theoretical results show that, surprisingly, \algAbbrev{}'s success in eliminating negative transfer does not depend on the quality of the proxy information, making \algAbbrev{} particularly useful when proxy information is weak or unreliable.
The extent of negative transfer depends instead on the quality of a pre-specified \textit{relevance function}.
We describe approaches to defining the relevance function in a purely source data-dependent way, and demonstrate application of these approaches in two synthetic examples and on a dataset of smoking behavior.

\section{PRELIMINARIES}\label{sec:prelim}
    \paragraph{Notation.}
        Vectors and matrices are denoted by bold lowercase letters: $\mathbf{a}_{i,j}$ is the entry in the $i^{\text{th}}$ row and $j^{\text{th}}$ column of $\mathbf{a}$.
        Sets are denoted by calligraphic font ($\mathscr{A}$), and $\mathscr{A}_i$ is the $i^{\text{th}}$ element of $\mathscr{A}$.
        We use $\mathbf{a}_{\mathscr{I}}$, where $\mathscr{I}$ is a set, to denote the subvector formed by selecting the elements of $\mathbf{a}$ at the indices in $\mathscr{I}$.
        Random variables are denoted by bold capital letters ($\mathbf{A}$), and the notation for probability distributions is subscripted by the corresponding random variable ($P_{\mathbf{A}}$).
        For instance, $\mathbf{A}$ is the random variable with domain $\mathscr{A}$ and probability distribution $P_{\mathbf{A}}$.

    \paragraph{Bayesian transfer learning}
        is a general framework for leveraging data from source tasks to make predictions in a somewhat unrelated target task \citep{suder_bayesian_2023}.
        We consider a standard setting where tasks are characterized by both shared and task parameters.
        The learner has available to them source data $\sourcedata = \left[ \sourcedata_1, \ldots{}, \sourcedata_n \right]$ composed of stochastic observations $\sourcedata_i \in \datavalues$.\footnote{
            The source data matrix can equivalently be written $\sourcedata_{(1:n)}$ to make explicit that it is composed of all $n$ past observations.
            When referring to the source data matrix, we omit the subscript $(1:n)$, i.e., write $\sourcedata \equiv \sourcedata_{(1:n)}$.
        }
        We write the random variable characterizing the source data as $\sourcedataRV$.
        
        Each observation is generated in the context of a particular, possibly non-unique, \textit{task}.
        In the setting of \textit{unknown sources}, the learner need not be aware of which observations are generated in the context of the same, or similar, tasks.
        Below, we formalize this as a potential difference between the dependency structures characterizing the source data-generating process, on one hand, and the learner's model of the source data, on the other.
        
        The probability of each observation $\sourcedata_i$ depends both on \textbf{shared parameters} $\param \in \paramvalues$, which are the same for each task, and \textbf{task parameters} $\tsparam_i \in \tsparamvalues$, which differ between tasks.
        Given a $\left( \param, \tsparam_i \right)$, the learner can evaluate $p(\sourcedata_i \vert \param, \tsparam_i)$.
        As is typical in such formulations, we assume a single, data-generating value of each of the shared and source task parameters, which we denote $\targetconcept$ and $\truetsparam$, respectively.\footnote{
            As with the source data $\sourcedata$, the source task parameters can equivalently be written $\tsparam_{(1:n)}$.
            When referring to the source task parameters, we write $\tsparam \equiv \tsparam_{(1:n)}$.
        }
        \begin{definition}[Task]
            The $i^{\mathrm{th}}$ task specifies the distribution generating the $i^{\mathrm{th}}$ data point.
            It depends on the value of a \textbf{shared parameter} $\targetconcept$ and \textbf{task parameter} $\truetsparam_i$.
            The value $\targetconcept$ is assumed to be shared across all tasks, and so the task is equivalently identified by the value of $\truetsparam_i$.
        \end{definition}

        At deployment, the learner encounters an $(n + 1)^{\text{th}}$ task which will induce an observation $\targetdata$.
        Their goal is to predict $\targetdata$ on the basis of $\sourcedata$, which requires identification of the \textbf{target data-generating process}, i.e., of the shared parameter $\targetconcept$ and task parameter $\targettsparam$.
        
        The setting is visualized in \Cref{fig:setting}.
        Throughout, we implicitly depend on the following assumption:
        \begin{assumption}\label{as:theta-psi-ind}
            All dependencies in \Cref{fig:setting} are present in the data-generating process.
            The following dependencies are \textit{not} present in the data-generating process:
            \begin{enumerate}[label=(\alph*)]
                \item If $i \neq j$, $\sourcedata_i$ does not depend on $\tsparam_j$ \colorgrey{except possibly through $\tsparam_i$}.
                \item If $i \neq j$, $\sourcedata_i$ does not depend on $\sourcedata_j$ except through $\param$ \colorgrey{and possibly through $\tsparam_i$}.
                \item $\tsparam_{n+1}$ does not depend on $\param$.
                \item Proxy information $\auxInfo$ does not depend on $\param$ (see \Cref{sec:step1}).
            \end{enumerate}
        \end{assumption}
        In the setting of \textbf{unknown sources}, the learner does not have knowledge of the potential presence of the dependencies in grey text.
        We thus derive all quantities available to the learner (e.g., the likelihood of the source data given below) to reflect this absence of knowledge, i.e., omitting these potential dependencies.
        This results in a potential difference between the dependency structures characterizing the data-generating process and the learner's model of the data.
        
        \newcommand{\nodesize}{1.2cm}
        \newcommand{\nodesp}{.5cm}
        \begin{figure}[t!]
            \centering
            \begin{tikzpicture}
                \tikzset{minimum size=\nodesize}
                \node[style=circle,draw=black] (phi) {$\param$};
                \node[style=circle,draw=black,right=\nodesp of phi] (psi1) {$\tsparam_{(1:n)}$};
                \node[style=circle,draw=black,right=\nodesp of psi1] (tildepsi) {$\tsparam_{n+1}$};
                \node[style=circle,draw=black,fill=gray,below=\nodesp of psi1] (d1) {$\mathbf{d}_{(1:n)}$};
                \node[style=circle,draw=black,below=\nodesp of tildepsi] (tilded) {$\mathbf{d}_{n+1}$};
                \node[style=circle,draw=black,fill=gray,right=\nodesp of tildepsi] (z) {$\mathbf{z}$};

                \draw[->] (phi) -- (d1);
                \draw[->] (phi) -- (tilded);
                \draw[->] (psi1) -- (d1);
                \draw[->] (tildepsi) -- (tilded);
                \draw[->] (tildepsi) -- (z);
            \end{tikzpicture}
            \caption{Assumed dependencies between shared parameter $\param$, task parameters $\tsparam$, source data $\sourcedata \equiv \sourcedata_{(1:n)}$, target data $\sourcedata_{n+1}$, and proxy information $\auxInfo$.}
            \label{fig:setting}
        \end{figure}
        
        The Bayesian transfer learner assigns to values $\left( \param, \tsparam_{n+1} \right)$ a prior distribution, and so treats the parameters as random variables $\paramRV$ and $\targettsparamRV$ with distribution $P_{\paramRV,\targettsparamRV}$.
        For a possible value of the target data-generating process $\left( \param, \tsparam_{n+1} \right)$, the likelihood $L$ of the source data $\sourcedata$ is
        \begin{align}\label{eq:L}
            L(\sourcedata, \param, \tsparam_{n+1}) &\equiv p(\sourcedata \vert \param, \tsparam_{n+1}) = L(\sourcedata, \param)
        \end{align}
        because the target task parameter is not known by the learner to affect any of the source data (\Cref{as:theta-psi-ind}(a)).

        The probability of $\left( \param, \tsparam_{n+1} \right)$ under the posterior $P_{\paramRV, \targettsparamRV \vert \sourcedata}$ is
        \begin{align}
            p(\param, \tsparam_{n+1} \vert \sourcedata) &= \frac{L(\sourcedata, \param, \tsparam_{n+1}) ~ p(\param, \tsparam_{n+1})}{\E{\param^{\prime}, \tsparam_{n+1}^{\prime} \sim P_{\paramRV,\tsparamRV_{n+1}}}{L(\sourcedata, \param^{\prime}, \tsparam_{n+1}^{\prime})}} \nonumber \\
            &= \left( \frac{L(\sourcedata, \param) ~ p(\param)}{\E{\param^{\prime} \sim P_{\paramRV}}{L(\sourcedata, \param^{\prime})}} \right) p(\tsparam_{n+1})
        \end{align}
        where the second line follows from \Cref{as:theta-psi-ind}(a,c).

        As we show in \Cref{sec:negtransfer}, computing the likelihood in a classic way (described below) can lead to negative transfer.
        After describing the classic Bayesian learner's approach, we introduce a generic method to ``robustify'' the likelihood.
        \algAbbrev{} leverages this robustified method in its estimation of the predictive posterior.

    \paragraph{Classic Bayesian inference}
        additionally requires a prior over the source task parameters $P_{\tsparamRV}$.
        The posterior then marginalizes across this prior as follows:
        \begin{align}\label{eq:classic-L}
            p(\param, \tsparam_{n+1} \vert \sourcedata) &= \left( \frac{L(\sourcedata, \param) ~ p(\param)}{\E{\param^{\prime} \sim P_{\paramRV}}{L(\sourcedata, \param^{\prime})}} \right) p(\tsparam_{n+1}) \nonumber \\
            &= \left( \frac{\E{\tsparam \sim P_{\tsparamRV}}{L(\sourcedata, \param, \tsparam)} ~ p(\param)}{\E{\param^{\prime}, \tsparam \sim P_{\paramRV,\tsparamRV}}{L(\sourcedata, \param^{\prime}, \tsparam)}} \right) p(\tsparam_{n+1}) \nonumber \\
            &= p(\param \vert \sourcedata) ~ p(\tsparam_{n+1})
        \end{align}
        The prior $P_{\tsparamRV}$ encodes the learner's knowledge about the joint distribution of source task parameters.
        In the setting of unknown sources, the learner does not have knowledge of the dependencies between source task parameters (such as which observations arise from the same task, i.e., for which $\left( i, j \right)$ it is the case that $\tsparam^{\star}_i = \tsparam^{\star}_j$) and/or their prior may misrepresent the probability of encountering data generated under a given source task parameter value.

    \paragraph{Likelihood weighting}
        is a technique whereby the learner specifies a vector of weights $\weight$ that determines the contribution of each observation to the overall weighted likelihood \citep{grunwald_safe_2011}.
        When some $\weight_i > \weight_j$, it can be seen as increasing the influence of the $i^{\mathrm{th}}$ data point relative to the $j^{\mathrm{th}}$ data point.

\section{A BAYESIAN PERSPECTIVE ON NEGATIVE TRANSFER}\label{sec:negtransfer}
    Negative transfer refers to the phenomenon that learning from source data can hurt performance in the target task \citep{wang_characterizing_2019}.
    Here, we give a formal statement of the Bayesian transfer learner's objective which will allow us to make a precise and interpretable statement about when negative transfer will occur.

    The Bayesian transfer learner's goal is to identify the target data-generating process.
    Since the effects of the task parameters will not transfer, the source data can help the Bayesian learner identify the target data-generating process only insofar as it identifies the shared parameter.
    This objective has a natural information-theoretic interpretation, given in \Cref{def:pi-tig}: The \textbf{information gain}, or degree to which a Bayesian learner has ``gained information'' about $\param$, is the expected log ratio of the posterior to prior odds of $\param$.
    Information gain measures are applied in contexts like experimental design \citep{rainforth_modern_2024} and model selection \citep{oladyshkin_connection_2024}.
    
    Because we are interested in the learner's information gains under the true data-generating process, we define the information gain as an expectation across the true distribution of source data.
    To reduce notational clutter, we use $\truedataRV$ to refer to the random variable $\sourcedataRV \vert \targetconcept, \truetsparam$, which follows the distribution of source data under the true data-generating parameters $\left( \targetconcept, \truetsparam \right)$ (which are unavailable to the learner).
    \begin{definition}[Information gained by the classic Bayesian learner $\cIG$]\label{def:pi-tig}
        The information gained by the classic Bayesian transfer learner about the shared parameter $\targetconcept$ is
        \begin{align}
            \cIG\left( \targetconcept \right) &\equiv \E{\sourcedata \sim P_{\truedataRV}}{\log{\left( \frac{p(\targetconcept \vert \sourcedata)}{p(\targetconcept)} \right)}}. \nonumber
        \end{align}
    \end{definition}
    If $\cIG\left( \targetconcept \right) > 0$, the learner has successfully transferred information about $\targetconcept$ from the source to target data ($p(\targetconcept \vert \sourcedata) > p(\targetconcept)$, i.e., they prefer $\targetconcept$ after viewing data).
    Otherwise, they are worse off than before ($p(\targetconcept \vert \sourcedata) \leq p(\targetconcept)$, i.e., they preferred $\targetconcept$ before viewing data).
    We define positive and negative transfer as:
    \begin{definition}[Positive and negative transfer]\label{def:transfer}
        The classic Bayesian learner experiences \text{\normalfont positive transfer} when $\cIG\left( \targetconcept \right) > 0$; otherwise, they experience \text{\normalfont negative transfer}.
    \end{definition}

    We now provide a result showing that the threat of negative transfer is affected by the reliability of the prior over source task parameters $P_{\tsparamRV}$.
    The key quantity is a measure of likelihood misspecification:
    \begin{definition}[Misspecification of the classic likelihood $\cDelta$]
        The degree to which the classic likelihood is misspecified is $\cDelta \equiv \kld{P_{\truedataRV}}{P_{\sourcedataRV \vert \targetconcept}}$ where $\mathrm{D_{KL}}$ is the Kullback-Leibler divergence measure.
    \end{definition}
    In the presence of negative transfer, likelihood misspecification increases with the misspecification of the prior over source task parameters (\citealt{sloman_bayesian_2024} Theorem 4.11).
    To see this, recall from \Cref{eq:classic-L} that the density for $P_{\sourcedataRV \vert \targetconcept}$ marginalizes across $P_{\tsparamRV}$.

    \Cref{prop:ig-classic} shows that $\cDelta$ is responsible for negative transfer.
    The proof, adapted from \citet{sloman_bayesian_2024}, is given in \Cref{ap:classic-ig}.
    It relies on the following assumption:

    \informalAssumption{\Cref{as:smoothness}}{The likelihood $L\left( \sourcedata, \param \right)$ is ``smooth enough'' in a neighborhood of $\targetconcept$.
        The formal condition is given in \Cref{ap:classic-ig}.
    }

    \begin{theorem}[Negative transfer with a classic likelihood (modified from \citealt{sloman_bayesian_2024} Theorem 4.5)]\label{prop:ig-classic}
        Under \Cref{as:smoothness} in \Cref{ap:classic-ig},
        \begin{align}\label{eq:prop-ig-class-bound}
            \cIG\left( \targetconcept \right) \leq A \left( B - \cDelta \right)
        \end{align}
        where $A$ and $B$ are constants that do not depend on $\cDelta$.
    \end{theorem}

    Because of its effect on $\cDelta$, the prior over source task parameters affects the risk of negative transfer.
    To remove the Bayesian transfer learner's dependence on this prior information, we introduce \algName{} (\algAbbrev{}).

\section{\algAbbrev{}}\label{sec:dalw}
    The Bayesian transfer learner faces two challenges: To make accurate predictions in the target task, they must (1) gain information about the target task parameter $\targettsparam$, which to their knowledge does not depend on the source data, and (2) avoid negative transfer, which as discussed in \Cref{sec:negtransfer} arises from a misspecified source task parameter prior $P_{\tsparamRV}$.
    Our proposed \algName{} (\algAbbrev{}) has three steps: First, to address challenge (1), \textit{proxy information} is used to form a posterior over the target task parameter $\tsparam_{n+1}$.
    Then, to address challenge (2), a \textit{relevance function} is used to construct a weighted likelihood for $\param$ that does not depend on any prior source task information.
    Finally, the learner combines their posterior over $\tsparam_{n+1}$ and weighted likelihood for $\param$ to form a robust posterior over the target data-generating process.
    The entire procedure is summarized in \Cref{alg:dalw}.
    
    The computational overhead required by \algAbbrev{} is comparable to that required by existing implementations of Bayesian inference. 
    As shown in \Cref{alg:dalw} \Cref{algstep:iter-start}--\Cref{algstep:iter-end} and discussed in \Cref{sec:step2}, the reweighting step (step 2) is performed in at most $T$ iterations, where $T$ is the user-supplied number of iterations for refinement of the relevance function.
    Once this step is performed, computation of the posterior is similar to other Bayesian inference methods.\footnote{
        As can be seen in the code we provide for the example predicting smoking behavior discussed in \Cref{sec:treatment-effect}, applying the relevance function requires minimal modifications to existing posterior update methods.
    }
    \algAbbrev{} thus increases computational overhead by at most the constant factor $T$.
    \begin{algorithm}[t!]
        \caption{\AlgName{} (\algAbbrev{})}\label{alg:dalw}
        \renewcommand{\algorithmicrequire}{\textbf{Input:}}
        \renewcommand{\algorithmicensure}{\textbf{Output:}}
        \begin{algorithmic}[1]
            \Require Source data $\sourcedata$, proxy information $\auxInfo$, prior $P_{\paramRV,\targettsparamRV}$, relevance function $\mathcal{R}$, and number of iterations for refinement of the relevance function $T$
            \Ensure R-weighted posterior predictive $P^{\mathcal{R}}_{\targetdataRV \vert \sourcedata,\auxInfo}$
            \State Compute $P_{\tsparamRV_{n+1} \vert \auxInfo}$ (\Cref{eq:step1}) \Comment{Step 1} \label{}
            \If{$\mathcal{R}$ depends on $P_{\paramRV}$} \Comment{Refinement of $\mathcal{R}$} \label{algstep:iter-start}
                \State $\widehat{P^{\mathcal{R}}_{\paramRV}} \gets P_{\paramRV}$
                \For{$t \in 1:T$}
                    \State Evaluate $\mathcal{R}$ using $\widehat{P^{\mathcal{R}}_{\paramRV}}$
                    \State Compute $\robustL$ (\Cref{eq:step2})
                    \State $\widehat{P^{\mathcal{R}}_{\paramRV}} \gets P^{\mathcal{R}}_{\paramRV \vert \sourcedata, \auxInfo}$
                \EndFor
            \EndIf \label{algstep:iter-end}
            \State Evaluate $\mathcal{R}$ using $\widehat{P^{\mathcal{R}}_{\paramRV}}$
            \State Compute $\robustL$ (\Cref{eq:step2}) \Comment{Step 2} \label{algstep:L}
            \State Compute $P^{\mathcal{R}}_{\targetdataRV \vert \sourcedata,\auxInfo}$ (\Cref{def:proxy-post-pred}) \Comment{Step 3} \label{}
        \end{algorithmic}
    \end{algorithm}

    \subsection{Step 1: Learning task parameters via proxies}\label{sec:step1}
        In step 1, the learner addresses the challenge of learning the target task parameter $\targettsparam$.
        We refer to information the learner has about the value of $\targettsparam$ and which does not depend on $\targetconcept$ (\Cref{fig:setting}) as \textbf{proxy information}. 
        We denote the proxy information $\auxInfo \in \auxvalues$.
        To leverage the proxy information to learn $\targettsparam$, the learner specifies a model for the likelihood of proxy information $\auxInfo$ given $\tsparam_{n+1}$, i.e., can compute $p(\auxInfo \vert \tsparam_{n+1})$.\footnote{
            In the absence of substantial prior knowledge about how the proxy information is generated, this model may be extremely expressive or even non-parametric.
        }
        Combined with the prior $P_{\targettsparamRV}$, this induces a distribution over $\auxInfo$.
        We denote the corresponding random variable $\auxInfoRV$.

        The posterior probability of a value $\tsparam_{n+1}$ is
        \begin{align}\label{eq:step1}
            p\left( \tsparam_{n+1} \vert \auxInfo \right) &= \frac{p\left( \auxInfo \vert \tsparam_{n+1} \right) ~ p\left( \tsparam_{n+1} \right)}{\E{\tsparam^{\prime}_{n+1}}{p\left( \auxInfo \vert \tsparam^{\prime}_{n+1} \right)}}.
        \end{align}

    \subsection{Step 2: Learning shared parameters via likelihood weighting}\label{sec:step2}
        In step 2, the learner addresses the challenge of avoiding negative transfer (learning the shared parameter $\targetconcept$ without depending on a source task parameter prior).
        
        Estimation of the target data-generating process requires estimating a joint distribution over both the shared and target task parameters $\left( \param, \tsparam_{n+1} \right)$.
        The challenge arises because the learner requires a model for $p\left( \sourcedata \vert \param, \tsparam_{n+1} \right)$.
        As we discussed in \Cref{sec:prelim}, using the classic likelihood of a value $\param$ requires marginalizing over possibly mistaken prior information about the source task parameters.
        
        In an ideal world, when computing $p\left( \sourcedata \vert \param, \tsparam_{n+1} \right)$ the learner would intervene on the source data and set $\truetsparam_1 = \ldots{} = \truetsparam_n = \tsparam_{n+1}$.
        While this is infeasible, the learner can perform a \textit{pseudo-intervention}: They can manipulate the source data to resemble the consequences of such an intervention.
        Using likelihood weighting techniques, the learner can reweight the data in order to assign higher weight to observations that are \textit{relevant} to the consequences of $\tsparam_{n+1}$.
        The probability of observing $\sourcedata_i$ if the $i^{\text{th}}$ task parameter had been ``set'' to $\tsparam_{n+1}$ is denoted $p\left( \sourcedata_i \vert \param, \tsparam_i = \tsparam_{n+1} \right)$.
        The probability of observing all source data in the task characterized by $\tsparam_{n+1}$ is denoted $p\left( \sourcedata \vert \param, \tsparam = \tsparam_{n+1} \right)$.

        Formally, the relevance of an observation is:
        \begin{definition}[Relevance $\csim{i}{\tsparam_{n+1}}$]\label{def:relevance}
            The relevance of the $i^{\mathrm{th}}$ observation to $\tsparam_{n+1}$ is computed by a \textbf{relevance function} $\mathcal{R}$ which is positively correlated with $p\left( \sourcedata_i \vert \targetconcept, \tsparam_i = \tsparam_{n+1} \right)$ in expectation with respect to $P_{\truedataRV,\targettsparamRV}$.
        \end{definition}

        Unlike the classic Bayesian transfer learner who uses the likelihood expression in \Cref{eq:L} to construct their posterior, \algAbbrev{} uses the \textbf{relevance- (r-)weighted likelihood} of each observation:
        \begin{align}\label{eq:step2}
            \robustL(\sourcedata_i, \param, \tsparam_i = \tsparam_{n+1}) \equiv p(\sourcedata_i \vert \param, \tsparam_i = \tsparam_{n+1})^{\csim{i}{\tsparam_{n+1}}}.
        \end{align}

        \Cref{tab:prompt} summarizes the differences between r-weighted and classic Bayesian inference.
        \begin{table}
            \centering
            \begin{tabular}{l|cc}
                & \multirow{2}{*}{Requires} & Negative transfer \\
                & & is due to \\
                \hline
                \multirow{2}{*}{Classic} & Source task prior & \multirow{2}{*}{Misspecified $P_{\tsparamRV}$} \\
                & $P_{\tsparamRV}$ & \\
                \hline
                \multirow{2}{*}{R-weighted} & Relevance function & \multirow{2}{*}{Low-fidelity $\mathcal{R}$} \\
                & $\mathcal{R}$ &
            \end{tabular}
            \caption{Key differences between classic and r-weighted Bayesian learning.
                \label{tab:prompt}
            }
        \end{table}
        The key idea of r-weighting is to substitute the requirement for accurate prior knowledge of the source task parameters with a requirement for a suitably-specified relevance function (i.e., the ability to anticipate the consequences of a pseudo-intervention on the source data).
        At first glance, this may appear to substitute a requirement for one form of prior knowledge with another.
        However, as we discuss below, specifying a suitable relevance function often does not require knowledge beyond that which is already encoded in the learner's model.

        \paragraph{Defining the relevance function.}
            \sloppy \Cref{def:relevance} requires that $\csim{i}{\tsparam_{n+1}}$ positively correlate with  $p\left( \sourcedata_i \vert \targetconcept, \tsparam_i = \tsparam_{n+1} \right)$.
            In \Cref{sec:theory}, we provide a result showing that the \textit{fidelity} of the relevance function --- the strength of this correlation --- affects the extent of the threat of negative transfer.
            However, computing $p\left( \sourcedata_i \vert \targetconcept, \tsparam_i = \tsparam_{n+1} \right)$ exactly would require access to $\targetconcept$, which the learner does not have.
            
            Given their ignorance of $\targetconcept$, one approach the learner could take would be to construct the relevance function in a way that depends only on their prior $P_{\paramRV}$, for instance, as
            \begin{align}\label{eq:iterR}
                \csim{i}{\tsparam_{n+1}} &= \E{\param \sim P_{\paramRV}}{p\left( \sourcedata_i \vert \param, \tsparam_i = \tsparam_{n+1} \right)}.
            \end{align}
            While \Cref{eq:iterR} will likely not exactly recover the consequences of the pseudo-intervention $\tsparam_i = \targettsparam$, in many cases of practical interest it will tend to correlate with $p\left( \sourcedata_i \vert \targetconcept, \tsparam_i = \tsparam_{n+1} \right)$.\footnote{
                See \Cref{ap:reweighted-rho} for discussion of a counterexample.
            }

            To correct for potential bias in $P_{\paramRV}$, we propose a procedure to iteratively refine the relevance function, outlined in \Cref{algstep:iter-start}--\Cref{algstep:iter-end} of \Cref{alg:dalw}.
            Notice that the source data, which the learner \textit{does} have access to, depend on $\targetconcept$, and so the learner can leverage these data to, for instance, refine the distribution across which the expectation in \Cref{eq:iterR} is taken.
            We propose the learner first evaluate the relevance function using $P_{\paramRV}$, then substitute $P_{\paramRV}$ in the definition of the relevance function with the resulting relevance-weighted posterior (\Cref{def:proxy-post} in \Cref{sec:step3}), reevaluate the relevance function, and repeat this process for a pre-specified number of iterations.
            In \Cref{sec:experiments}, we detail application of this iterative procedure in the context of two synthetic examples.
            While we observe that this procedure is effective in the context of these examples, an important direction for future work is establishing the conditions under which it converges, i.e., the conditions under which a relevance function satisfying \Cref{def:relevance} is available to the learner.
        
    \subsection{Step 3: Computing the r-weighted posterior predictive distribution}\label{sec:step3}
        We can now define the \textbf{relevance- (r-)weighted posterior} and \textbf{r-weighted posterior predictive distribution}.
        \begin{definition}[Relevance- (r-)weighted posterior distribution $P_{\paramRV, \targettsparamRV \vert \sourcedata, \auxInfo}^{\mathcal{R}}$]\label{def:proxy-post}
            The r-weighted posterior distribution $P_{\paramRV, \targettsparamRV \vert \sourcedata, \auxInfo}^{\mathcal{R}}$ is the distribution with density
            \begin{align}
                &p^{\mathcal{R}}(\param, \tsparam_{n+1} \vert \sourcedata, \auxInfo) = \nonumber \\
                &
                \frac{\robustL(\sourcedata, \param, \tsparam = \tsparam_{n+1}) ~ p\left( \auxInfo \vert \tsparam_{n+1} \right) ~ p\left( \param, \tsparam_{n+1} \right)}{\E{\param^{\prime}, \tsparam_{n+1}^{\prime} \sim P_{\paramRV, \tsparamRV_{n+1}}}{\robustL(\sourcedata, \param^{\prime}, \tsparam = \tsparam_{n+1}^{\prime}) ~ p\left( \auxInfo \vert \tsparam_{n+1}^{\prime} \right)}}. \nonumber
            \end{align}
        \end{definition}

        \begin{definition}[Relevance- (r-)weighted posterior predictive distribution $P_{\targetdataRV \vert \sourcedata, \auxInfo}^{\mathcal{R}}$]\label{def:proxy-post-pred}
            The r-weighted posterior predictive distribution $P_{\targetdataRV \vert \sourcedata, \auxInfo}^{\mathcal{R}}$ is the distribution with density
            \begin{align}
                p^{\mathcal{R}}(\targetdata \vert \sourcedata, \auxInfo) &= \E{\param, \tsparam_{n+1} \sim P_{\paramRV, \tsparamRV_{n+1} \vert \sourcedata, \auxInfo}^{\mathcal{R}}}{p(\targetdata \vert \param, \tsparam_{n+1})}. \nonumber
            \end{align}
        \end{definition}

\section{THEORETICAL RESULTS}\label{sec:theory}
    In \Cref{sec:negtransfer}, we introduced a formal framework for assessing the threat of negative transfer.
    In \Cref{sec:dalw}, we introduced a framework for Bayesian transfer learning that uses a pre-specified relevance function to r-weight the likelihood.
    Our goal here is to assess whether r-weighting can effectively reduce the threat of negative transfer, and if so, the conditions under which this is the case.

    To assess the threat of negative transfer to the r-weighted Bayesian transfer learner, we introduce an information gain measure analogous to \Cref{def:pi-tig}, but that measures the degree to which the r-weighted posterior favors $\targetconcept$ with respect to the prior:\footnote{
        See discussion in \Cref{ap:reweighted-ig} for interpretation of $P_{\auxInfoRV}$.
    }
    \begin{definition}[Information gained by the r-weighted Bayesian learner $\rIG$]\label{def:pir-tig}
        The information gained by the r-weighted Bayesian transfer learner about the shared parameter $\targetconcept$ is
        \begin{align}
            &\rIG\left( \targetconcept \right) \equiv \E{\sourcedata, \auxInfo \sim P_{\truedataRV, \auxInfoRV}}{\log{\left( \frac{p^{\mathcal{R}}(\targetconcept \vert \sourcedata, \auxInfo)}{p(\targetconcept)} \right)}}. \nonumber
        \end{align}
    \end{definition}

    Analogously to \Cref{def:transfer}, we say that the r-weighted Bayesian transfer learner experiences negative transfer when $\rIG\left( \targetconcept \right) \leq 0$.

    Below, we provide two results that together show that the relevance function controls the threat of negative transfer.
    \Cref{prop:ig-reweighted} shows that the threat of negative transfer to the r-weighted Bayesian transfer learner depends on misspecification of the r-weighted likelihood, where the misspecification can be interpreted as the degree to which the relevance function corrects for a mismatch between the source and possible target tasks.
    \Cref{thm:reweighting} decomposes this measure of misspecification, showing that it is a negative function of the \textit{fidelity} of the relevance function.
    The proofs of all results are deferred to \Cref{ap:math}.

    Misspecification of the r-weighted likelihood is:
    \begin{definition}[Misspecification of the r-weighted likelihood $\rDelta$]
        The degree to which the r-weighted likelihood is misspecified is $\rDelta \equiv \E{\tsparam_{n+1} \sim P_{\targettsparamRV}}{\kld{P_{\truedataRV}}{P_{\sourcedataRV^{\mathcal{R}\left( \tsparam_{n+1} \right)} \vert \targetconcept, \tsparam = \tsparam_{n+1}}}}$
        where $P_{\sourcedataRV^{\mathcal{R}\left( \tsparam_{n+1} \right)}}$ is the distribution of data resulting from viewing $\csim{i}{\tsparam_{n+1}}$ replicates of each $\sourcedata_i$.
    \end{definition}

    In the r-weighted case, the misspecification stems from the failure of the pseudo-replication to correct for a mismatch in the source tasks (the consequences of $\tsparam^{\star}$) and possible target tasks (the consequences of possible values $\tsparam_{n+1}$).

    \Cref{prop:ig-reweighted} gives a result analogous to \Cref{prop:ig-classic} for the r-weighted case.
    It depends on the following assumptions:
    \begin{assumption}[$\robustL$ is bounded]\label{as:jensen-gap}
        The r-weighted likelihood $\robustL\left( \sourcedata, \param, \tsparam = \tsparam_{n+1} \right)$ is bounded from both below and above: $\exists a, b \in \mathbb{R}^+$ such that $\forall \sourcedata \in \mathscr{D}, \param \in \paramvalues, \tsparam_{n+1} \in \tsparamvalues$, $a \leq \robustL\left( \sourcedata, \param, \tsparam = \tsparam_{n+1} \right) \leq b$.
    \end{assumption}
        
    \informalAssumption{\Cref{as:convergence}}{The proxy is sufficiently informative in the sense that the ``variability'' of $\targettsparamRV \vert \auxInfo$ is smaller than the ``variability'' of $\targettsparamRV$ by a ``large enough'' margin.
        The formal condition is given in \Cref{ap:reweighted-ig}.
    }

    \informalAssumption{\Cref{as:smoothness-reweighted}}{The r-weighted likelihood $\robustL\left( \sourcedata, \param, \tsparam = \tsparam_{n+1} \right)$ is ``smooth enough'' in a neighborhood of $\targetconcept$ \textbf{and} the estimated relevances are not ``too large''.
        The formal condition is given in \Cref{ap:reweighted-ig}.
    }

    \begin{theorem}[Negative transfer with an r-weighted likelihood]\label{prop:ig-reweighted}
        Under \Cref{as:jensen-gap} and \Cref{as:convergence,as:smoothness-reweighted} in \Cref{ap:reweighted-ig},
        $$\rIG\left( \targetconcept \right) \leq A\left(C - \rDelta \right)$$
        where $A$ and $C$ are constants that do not depend on $\rDelta$.
    \end{theorem}

    \Cref{thm:reweighting} analyzes the effect of $\mathcal{R}$ on $\rDelta$.
    The role of $\mathcal{R}$ in mitigating negative transfer depends on the \textit{fidelity} of the relevance function:   
    
    \informalAssumption{\Cref{def:rho}}{$\rho^{\mathcal{R}}$ is a measure of the fidelity of the relevance function, i.e., the extent of the correlation of $\csim{i}{\tsparam_{n+1}}$ with $p(\sourcedata_i \vert \targetconcept, \tsparam_i = \tsparam_{n+1})$ in expectation with respect to $P_{\truedataRV,\tsparamRV_{n+1}}$.
        The formal definition is given in \Cref{ap:reweighted-rho}.
    }

    \begin{proposition}[Negative transfer is reduced by high-fidelity relevance functions]\label{thm:reweighting}
        $\rDelta$ is a negative function of $\rho^{\mathcal{R}}$.
        In particular,
        \begin{align}
            \rDelta = \E{}{\mathrm{ESS}\left( \sourcedata, \tsparam_{n+1} \right) \mathrm{DIS}\left( \sourcedata, \tsparam_{n+1}\right)} - n \rho^{\mathcal{R}} + D \nonumber
        \end{align}
        where $\mathrm{ESS}\left( \sourcedata, \tsparam_{n+1} \right) \equiv \sum_{i=1}^n \csim{i}{\tsparam_{n+1}}$ is the effective sample size induced by the relevance function $\mathcal{R}$ evaluated on the sample $\sourcedata$ and task parameter $\tsparam_{n+1}$, $\mathrm{DIS}\left( \sourcedata, \tsparam_{n+1} \right) \equiv - \log{\left( p\left( \sourcedata \vert \targetconcept, \tsparam = \tsparam_{n+1} \right) \right)}$ captures the dissimilarity of the source data to the target task characterized by $\tsparam_{n+1}$, the expectation is taken with respect to $P_{\truedataRV,\tsparamRV_{n+1}}$, and the constant $D$ does not depend on $\mathcal{R}$.
    \end{proposition}

    \begin{remark}[Weakly informative proxies mitigate negative transfer]
        \Cref{thm:reweighting} shows that $\rDelta$ does not depend on the accuracy of the learner's inferences about $\targettsparam$, i.e., on the informativeness of the proxy information.
        Informative proxies facilitate targeted inference insofar as they facilitate estimation of the target task parameter, but do not improve the r-weighted learner's ability to recover the shared parameter from the source data.
        \algAbbrev{}'s provable advantage over classic Bayesian inference does depend on the availability of some proxy information only to satisfy \Cref{as:convergence}, required in the proof of \Cref{prop:ig-reweighted}: If the available proxy information is not somewhat informative, the magnitude of $\rDelta$ does not necessarily imply the degree of the threat of negative transfer.
    \end{remark}

\section{EXAMPLES}\label{sec:experiments}
    We here demonstrate application of \algAbbrev{} in two synthetic settings and on one real-world dataset.
    Additional details of all examples are provided in \Cref{ap:experiments}.

    Taken together, these examples demonstrate that \algAbbrev{} can significantly reduce the threat of negative transfer, and that its effectiveness in doing so is robust to unreliable and misleading proxy information.
    In all examples, we defined the relevance function in a purely source data-dependent way, illustrating the availability of effective relevance functions in settings of practical interest.

    \subsection{Treatment effect estimation}\label{sec:treatment-effect}
        To continue with our motivating example, we first demonstrate application of \algAbbrev{} to treatment effect estimation using similar modeling paradigms to those used in clinical prediction tasks \citep{gunn-sandell_practical_2023}.
        We first apply \algAbbrev{} in a synthetic setting that allows us to manipulate factors like the risk of negative transfer.
        We then apply \algAbbrev{} to a real-world dataset of smoking behavior.

        In both cases, we consider the treatment effect to be transferable, i.e., the parameter corresponding to the size of the influence of the treatment on outcomes is shared across tasks.
        Here, negative transfer refers to a situation where learning from the source data causes the learner to believe that the treatment has an effect opposite to its true effect (e.g., a negative rather than positive treatment effect).
    
        \paragraph{Linear regression.}
            The synthetic data in this example are generated according to the model
            $$\sourcey_i \vert \sourcex_i \sim \mathcal{N}\left( \targetconcept \sourcex_{i,1} + \truetsparam_i \sourcex_{i,2}, 1 \right)$$
            where the shared parameter $\targetconcept$ represents the effect of a synthetic treatment $\sourcex_{i,1}$ and $\truetsparam_i$ represents the effect of a synthetic confounder (e.g., quality of care) $\sourcex_{i,2}$.
            We computed the relevances as $\csim{i}{\tsparam_{n+1}} \propto \E{\param \sim \widehat{P^{\mathcal{R}}_{\paramRV}}}{p\left( \sourcedata_i \vert \param, \tsparam_i = \tsparam_{n+1} \right)}$ where $\widehat{P^{\mathcal{R}}_{\paramRV}}$ was formed using the iterative procedure described in \Cref{sec:step2}.

            \noindent \textit{Alleviating the risk of negative transfer:}
            \Cref{fig:betabinom} shows how $\rIG$ compares with $\cIG$ as a function of the risk of negative transfer and the representativeness of the target task in the distribution of source tasks.
            \begin{figure}[t!]
                \begin{subfigure}{\linewidth}
                    \includegraphics[width=\linewidth]{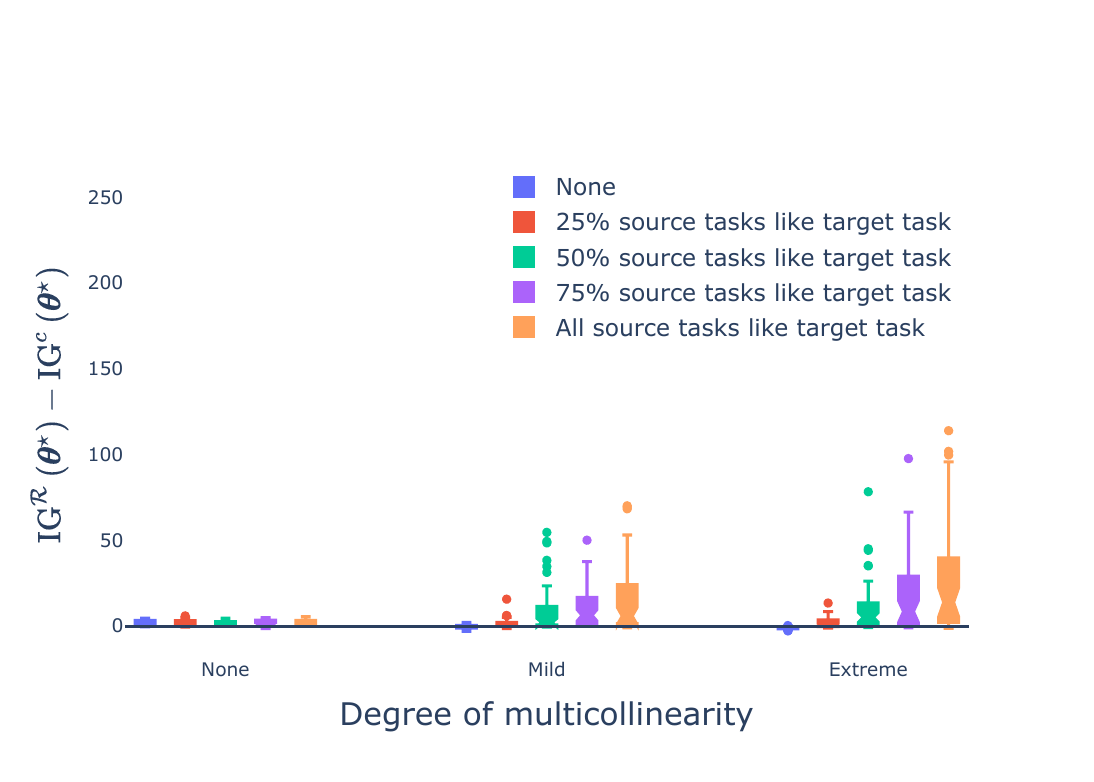}
                    \caption{\label{fig:betabinom}
                        Alleviating the risk of negative transfer.
                        In all cases, proxy information is uncontaminated.
                    }
                \end{subfigure} \hfill \begin{subfigure}{\linewidth}
                    \includegraphics[width=\linewidth]{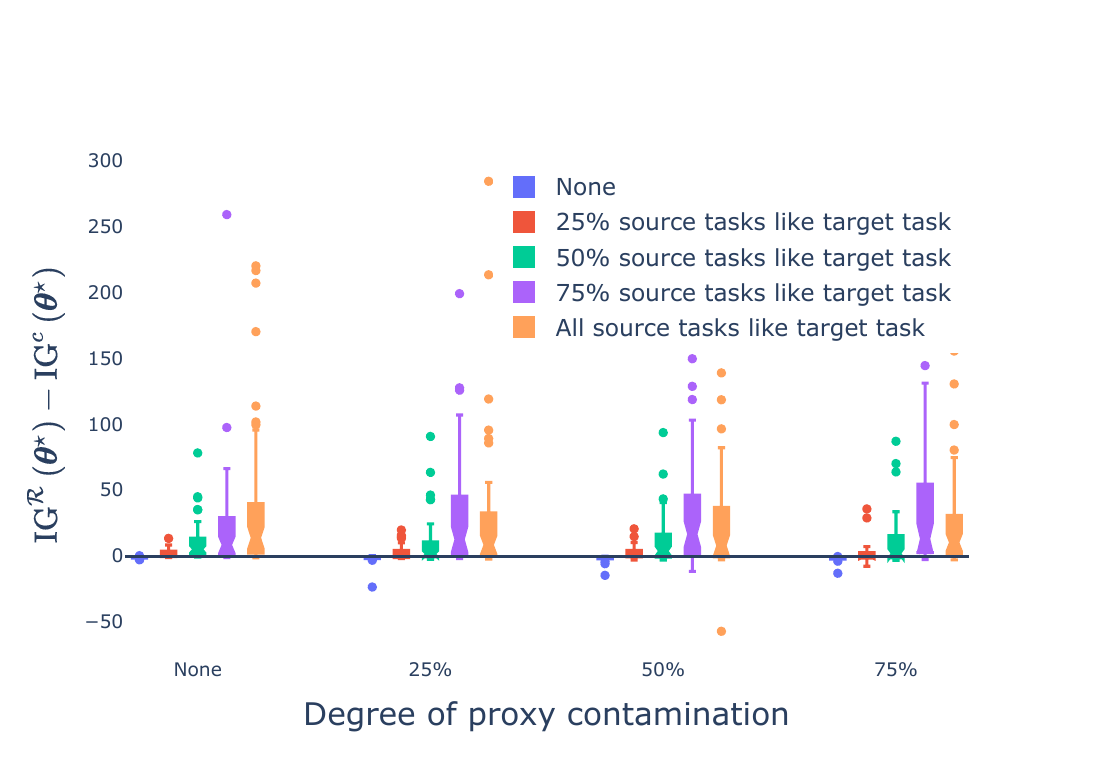}
                    \caption{\label{fig:linear-contamination}
                        Robustness to noisy proxy information.
                        In all cases, there is an extreme degree of multicollinearity (i.e., large threat of negative transfer).
                    }
                \end{subfigure}
                \caption{Advantage of the r-weighted learner in the linear regression setting.
                    Each box includes results from 50 simulations and shows the interquartile region (boxes) and outliers (points) of $\rIG\left( \targetconcept \right) - \cIG\left( \targetconcept \right)$.
                    In each simulation, $\targettsparam$, $\sourcedata$, and $\auxInfo$ are randomly regenerated.
                }
            \end{figure}
            To induce the risk of negative transfer, we manipulated the degree of multicollinearity between $\sourcex_{(\cdot,1)}$ and $\sourcex_{(\cdot,2)}$: More multicollinearity makes $\targetconcept$ and $\targettsparam$ harder to separately identify, so we interpret this as a higher risk of negative transfer.
            We also varied the distribution of source tasks.
            When $p\%$ of tasks resemble the target task, $1-p\%$ of tasks are set to a value that is well-represented by $P_{\tsparamRV}$.
            In this sense, our results are a somewhat conservative test of \algAbbrev{}.

            When there is no multicollinearity, the classic learner is not at risk of negative transfer, and performs on par with the learner with an r-weighted likelihood.
            When all source tasks are well-represented in the learner's prior (blue box), the classic learner's prior is well-specified and they perform on par with the r-weighted learner.
            When there is a risk of negative transfer, $\rIG$ is generally higher, especially when many source tasks resemble the target task (and the classic learner's prior $P_{\tsparamRV}$ is more misspecified).

            \noindent \textit{Robustness to noisy proxy information:}
            The synthetic proxy information represents feedback from a domain expert.
            While domain experts may not be able to articulate precise knowledge of the target task, they can often provide intuitive assessments \citep{kahneman_conditions_2009} such as the degree to which an outcome is representative of a given situation \citep{tversky_judgment_1974}.
            Our synthetic proxy represents a domain expert who is presented with a hypothetical outcome and asked the degree to which it is representative of the target task on a scale of $0$--$7$.
            To assess the robustness of \algAbbrev{} to noisy proxies, we contaminated a percentage of these synthetic judgments.
            The percentage of contaminated proxy values is unknown to the learner, who always models the proxy information as completely uncontaminated.

            In line with our result in \Cref{thm:reweighting}, \Cref{fig:linear-contamination} shows that noisy proxy information does not affect the relative advantage of the r-weighted learner: Regardless of the degree of proxy contamination, the r-weighted learner tends to outperform the classic learner.

        \paragraph{Predicting smoking behavior.}
            We also applied \algAbbrev{} to predict smoking behavior in a dataset from \citet{hasselblad_meta-analysis_1998} provided by the R package \texttt{netmeta} \citep{balduzzi_netmeta_2023}, which consists of data from 24 studies on the number of patients who stopped smoking after receiving one of four treatments.\footnote{This example was inspired by the example detailed in \citet{holzhauer_network_2025}.
                The code used the package \texttt{brms} \citep{burkner_brms_2017} and Stan modeling language \citep{stan_stan_2024}.
            }
            Each study includes data from patients who received some but not all treatments.
            We considered each study a separate task.
            Each observation is indexed by study and treatment (so $\sourcey_i$ is the number of patients who stopped smoking after receiving a given treatment in a given study and $\tsparam_i = \tsparam_j$ if $i$ and $j$ index data from different treatments administered as part of the same study).
            We modeled the data as
            $$\sourcey_i \vert \sourcex_i, \param, \tsparam_i \sim \mathrm{Binomial}\left( \mathrm{sigmoid}\left( \param \sourcex_{i,(1:4)}^{\top} + \tsparam_i \right), N_i \right)$$
            where $\sourcex_{i,(1:4)}$ are indicators of the treatment received and $N_i$ is the number of patients who received the indicated treatment in the indicated study.
            We considered 24 different partitions of the data into source and target data, with each partition treating data from one study as target data and data from the remaining 23 studies as source data.
            We defined the relevance function as
            $$\csim{i}{\tsparam_{n+1}} = \mathrm{sigmoid}\left( n \frac{p\left( \sourcedata_i \vert \param = 0, \tsparam_i = \tsparam_{n+1} \right)}{p\left( \sourcedata \vert \param = 0, \tsparam = \tsparam_{n+1} \right)} \right).$$

            \begin{figure}[t!]
                \centering
                \includegraphics[width=\linewidth]{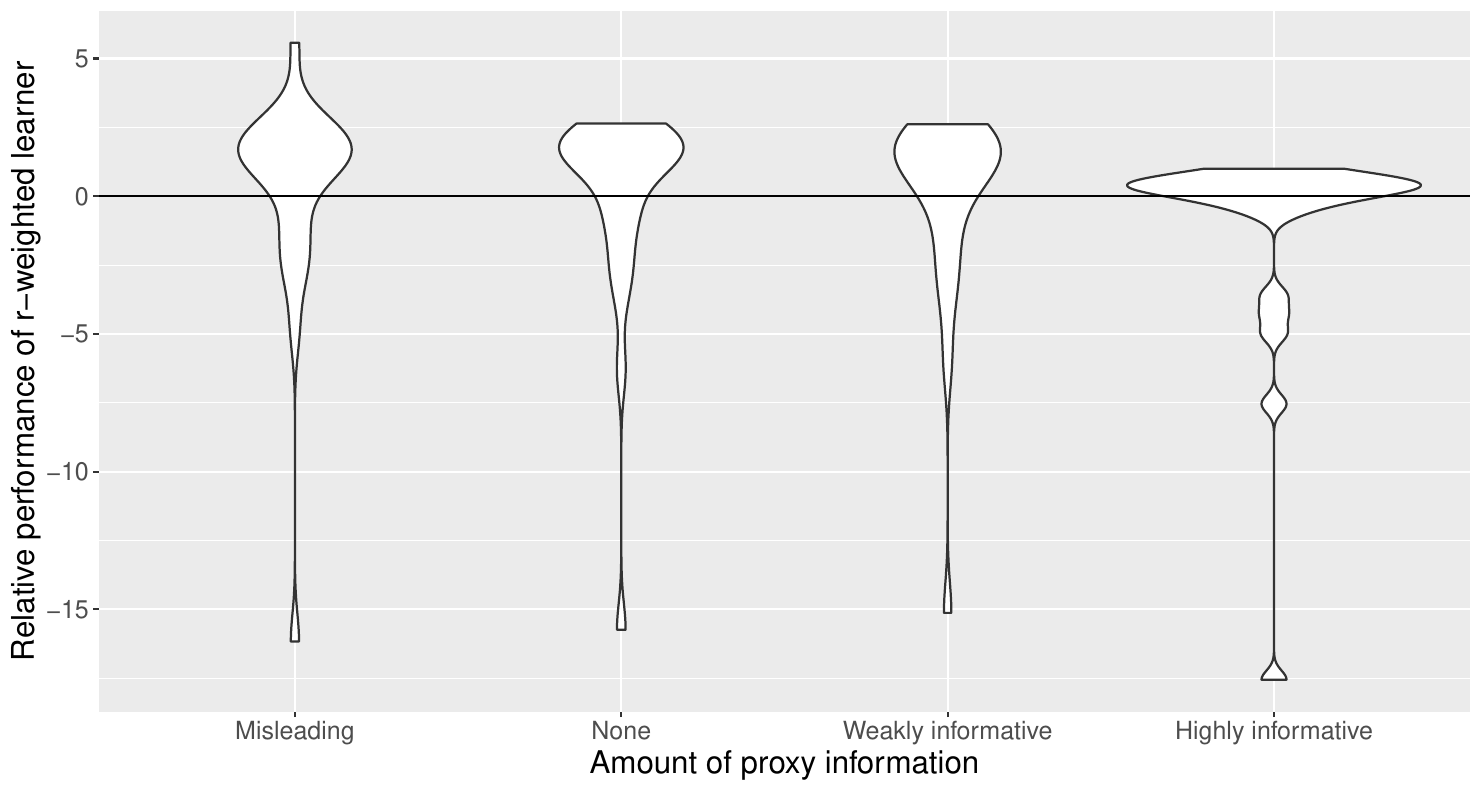}
                \caption{\label{fig:smoking}
                    Advantage of the r-weighted learner in the dataset of smoking behavior.
                    Each plot shows the distribution of values of $\log\left( \frac{p^{\mathcal{R}}\left( \sourcedata_{n+1} \vert \sourcedata, \auxInfo \right)}{p\left( \sourcedata_{n+1} \vert \sourcedata, \auxInfo \right)} \right)$ across 24 partitions of source/target data.
                }
            \end{figure}
            \Cref{fig:smoking} shows the relative performance of the r-weighted and classic Bayesian transfer learners as a function of the informativeness of the proxy information.
            Unlike in our synthetic example, here we do not have access to the true value $\targetconcept$ and so cannot directly compute $\rIG\left( \targetconcept \right)$ and $\cIG\left( \targetconcept \right)$.
            Instead, we assess how well the two methods can predict the outcome in the target task.
            The classic learner here has the advantage of prior source information in the form of knowledge of which observations belong to the same task.
            We also do not anticipate a substantial threat of negative transfer here.
            Nevertheless, the r-weighted learner outperforms the classic learner in the majority of cases.

            \noindent \textit{Robustness to noisy proxy information:}
            The synthetic proxy information represents an imprecise estimate of the value of $\targettsparam$.
            Highly informative proxy information refers to more precise estimates than weakly informative proxy information.
            To assess the robustness of PROMPT to misleading proxies, we added a bias to some synthetic estimates.
            While the learner is aware of the degree of precision of an estimate, they are unaware of the potential presence of bias.
                
            \Cref{thm:reweighting} showed that the extent of negative transfer for the r-weighted learner does not depend on the informativeness of the proxy information.
            \Cref{fig:smoking} shows that the r-weighted learner's advantage appears to actually decrease with the informativeness of the proxy information.
            Further inspection revealed that, in this example, the r-weighted learner's performance is not sensitive to the amount of proxy information, and the difference reflects the classic learner's higher performance in the presence of more informative proxy information.\footnote{This difference between the r-weighted and classic learners' sensitivity to proxy information is largely accounted for by an effective difference in model structure; as we describe in \Cref{ap:smoking}, the r-weighted learner learns a model with a single intercept rather than a separate linear effect for each study.}
            Further understanding this insensitivity to proxy information, as well as the nature of the tasks that lead the r-weighted learner to perform much worse than the classic learner, is a direction for future investigation.

    \subsection{Gaussian process regression}\label{sec:gp}
        We next demonstrate application of \algAbbrev{} in a Gaussian Process (GP) regression setting with a composite kernel.
        Additional details and results are given in \Cref{ap:gp}.
        Data were generated according to the model
        $$\sourcey_i \vert \sourcex_i \sim \mathcal{GP}\left( 0, \mathbf{k}(\sourcex_i, \sourcex) \right)$$
        where $\mathbf{k}(\sourcex_i, \sourcex) = \mathrm{RBF}_{\param}(\sourcex_i, \sourcex) + \mathrm{RBF}_{\tsparam_i}(\sourcex_i, \sourcex)$ and $\mathrm{RBF}_l$ is the radial basis function with lengthscale $l$.\footnote{The kernel was renormalized to have an amplitude of 1.}
        This setting poses a risk of negative transfer because the shared and task parameters act in combination to determine the smoothness of the sampled functions \citep{sloman_bayesian_2024}.
        We used the same methods to generate proxy information and specify the relevance function as for the linear regression example, but varied the number of iterations used for refinement of the relevance function.
        \Cref{fig:gp-supp} in \Cref{ap:gp} shows how $\rIG$ compares with $\cIG$ as a function of each of multiple simulation parameters.
        In all cases, the r-weighted learner tends to identify the value of the shared parameter as well as, and usually more successfully than, the classic learner.

        \noindent \textit{Robustness to noisy proxy information:}
        \Cref{fig:gp-proxy-noise} in \Cref{ap:gp} shows how $\rIG$ compares with $\cIG$ as a function of the amount of proxy contamination.
        In line with our result in \Cref{thm:reweighting}, the r-weighted learner outperforms the classic Bayesian learner even in the presence of substantial proxy contamination (although their advantage is greatest in the absence of proxy contamination).

\section{DISCUSSION}
    We presented a Bayesian perspective on negative transfer, from which we showed that negative transfer can arise from misspecified prior source information.
    Based on this insight, we developed \algAbbrev{}, a novel framework for Bayesian transfer learning which alleviates the learner's dependence on prior source information.
    The framework of \algAbbrev{} can accommodate a variety of relevance functions and forms of proxy information.
    \algAbbrev{}'s provable advantage depends on the fidelity of the specified relevance function.
    In \Cref{sec:experiments}, we provided concrete examples of possible relevance functions and demonstrated \algAbbrev{}'s robustness to noisy and misleading proxies in a variety of settings.
    We found that in practice we were able to specify relevance functions of sufficiently high fidelity to reduce negative transfer.
    Ultimately, however, there may exist situations where such a relevance function is unavailable (for example, if $\param$ and $\tsparam$ interact such that the direction of the gradient of predictions with respect to $\tsparam$ depends on $\param$).
    As \Cref{tab:prompt} shows, in such cases, the practitioner must make a choice about whether they can more confidently specify the prior over source task parameters or the relevance function.
    The development of a more systematic framework for defining the relevance function is a promising avenue for future work.
    Many transfer learning applications leverage high-dimensional, non-linear datasets \citep{suder_bayesian_2023} and future work should in particular look to the development of a scalable framework for applying and evaluating \algAbbrev{} in such contexts.

\begin{acknowledgements}
    The authors thank Ayush Bharti and Sammie Katt for helpful feedback on an initial draft, and several anonymous reviewers for helpful comments.
    This work was supported by the Research Council of Finland Flagship programme: Finnish Center for Artificial Intelligence FCAI and decisions 358958, 359567.
    SJS and SK were supported by the UKRI Turing AI World-Leading Researcher Fellowship, [EP/W002973/1].
    This work used the Computational Shared Facility at The University of Manchester.
\end{acknowledgements}

\bibliography{bibliography}
\emptythanks
\newcounter{footnotecounter}
\setcounter{footnotecounter}{\value{footnote}}
\setcounter{footnote}{0}
\onecolumn
\title{\ourTitle\\(Supplementary Material)}
\appendix
\bibliographystyle{plainnat}

\maketitle

The appendix is organized as follows:
\begin{itemize}
    \item In \Cref{sec:related-work}, we discuss related works in more detail.
    \item In \Cref{ap:math}, we provide proofs of all our mathematical results.
    \item In \Cref{ap:experiments}, we provide further details of the examples described in \Cref{sec:experiments}.
        \Cref{ap:gp} additionally provides the results of the GP regression example.
\end{itemize}

\section{RELATED WORK}\label{sec:related-work}
    \paragraph{Likelihood weighting}
        has been applied for purposes that include potential model misspecification \citep{grunwald_safe_2011,miller_robust_2019,dewaskar_robustifying_2023}, potential conflation of transferable and task-specific effects \citep{ibrahim_power_2000,ibrahim_optimality_2011,ibrahim_power_2014,suder_bayesian_2023}, model selection \citep{ibrahim_power_2014}, and increased efficiency of MCMC samplers~\citep{schuster_markov_2021}.

    \paragraph{Probabilistic meta-learning}
        \citep{gordon_meta-learning_2019} is a paradigm in which a meta-learner simultaneously learns a transferable parameter value and a distribution over task parameter values.
        Unlike \algAbbrev{}, this framework assumes the data sources are known in the sense that each data point can be indexed by its task.
        This distinction also sets us apart from other Bayesian meta-learning approaches \citep{grant_recasting_2018,yoon_bayesian_2018,patacchiola_bayesian_2020}.
        Moreover, the aim of probabilistic meta-learning is to learn a distribution over task parameters.
        When the target task will arise from the same distribution as the source tasks, probabilistic meta-learning facilitates good performance on average across tasks.
        However, the goal of \algAbbrev{} is to provide a posterior predictive distribution tailored to a target task that may not arise from the same distribution as the source tasks.

    \paragraph{Using domain similarity for domain adaptation.}
        Many existing theoretical bounds for domain adaptation rely on the similarity between source and target tasks \citep{redko_survey_2019}.
        Some approaches to domain adaptation use similarity of covariates (inputs) in the target and source tasks to weight source data during training \citep{plank_effective_2011,ponomareva_biographies_2012,remus_domain_2012,ruder_learning_2017} or importance sampling techniques \citep{quinonero-candela_dataset_2009}.
        While this can be effective in cases of pure covariate shift (a change in the distribution of inputs), our formulation allows for differences in the map between covariates and outcomes that cannot be detected on the basis of covariate information alone.

    \paragraph{Proximal causal learning}
        is a paradigm that uses proxy information to learn causal effects \citep{kuroki_measurement_2014,tchetgen_introduction_2020,alabdulmohsin_adapting_2023,tsai2024proxy}.
        Our setting is similar to the multi-domain adaptation setting of \citet{tsai2024proxy}.
        We differ in that (i) we assume data sources are unknown, while they assume data can be indexed by its task, and (ii) we assume the presence of both shared and task parameters, while they do not distinguish between these.
        While our method for estimating the task parameter also leverages proxy methods, we differ in our usage of reweighting methods to estimate the shared parameter, which facilitates robust estimation without requiring additional proxy information.

        As shown in \Cref{fig:setting}, our formulation is stated in terms of the dependencies between shared parameters, task parameters, and observations, and so our work shares conceptual connections with the more general paradigm of causal inference.
        For instance, conceptualizing r-weighting as a pseudo-intervention requires conceptualizing the task parameters as a cause of the observations.
        We do not however require that either the shared or task parameters parameterize the causal effect of one observable variable on another; these parameters can represent any unobservable factor influencing the data.

    \paragraph{Human-in-the-loop learning.}
        In many applications, domain experts are a viable source of proxy information, and so our work can be tied to human-in-the-loop machine learning \citep{wu_survey_2022}.
        Like us, some human-in-the-loop methods leverage expert feedback in a Bayesian framework. 
        For example,~\cite{nahal_human-in-the-loop_2024} use expert feedback for learning in out-of-distribution settings, while
        \cite{sundin_improving_2018} query experts about the relevance of a given feature for outcome prediction.

\section{MATHEMATICAL DETAILS}\label{ap:math}
    \subsection{Definitions}
        \begin{itemize}
            \item $\ent{P}$ is the entropy of distribution $P$ with density $p$:
                $$\ent{P} = - \E{\sourcex \sim P}{\log{\left( p(\sourcex) \right)}}$$
            \item $\crossent{P}{Q}$ is the cross-entropy from distribution $P$ to distribution $Q$ with density $q$:
                $$\crossent{P}{Q} = - \E{\sourcex \sim P}{\log{\left( q(\sourcex) \right)}}$$
            \item $\kld{P}{Q}$ is the Kullback-Leibler divergence from distribution $P$ with density $p$, to distribution $Q$ with density $q$:
                $$\kld{P}{Q} = \E{\sourcex \sim P}{\log{\frac{p(\sourcex)}{q(\sourcex)}}}$$
        \end{itemize}

    \subsection{Proof of Theorem 2.4}\label{ap:classic-ig}
        The information gain achieved by the classic Bayesian learner (\Cref{def:pi-tig}) can be written as:
        \begin{align}
            \cIG\left( \targetconcept \right) &= \E{\sourcedata \sim P_{\truedataRV}}{\log{\left( \frac{p(\targetconcept \vert \sourcedata)}{p(\targetconcept)} \right)}} \nonumber \\
            &= \E{\sourcedata \sim P_{\truedataRV}}{\log{\left( \frac{\frac{L\left( \sourcedata, \targetconcept \right) ~ p\left( \targetconcept \right)}{\E{\param \sim P_{\paramRV}}{L\left( \sourcedata, \param \right)}}}{p\left( \targetconcept \right)} \right)}} \nonumber \\
            &= \E{\sourcedata \sim P_{\truedataRV}}{\log{\left( \frac{L\left( \sourcedata, \targetconcept \right)}{\E{\param \sim P_{\paramRV}}{L\left( \sourcedata, \param \right)}} \right)}} \nonumber
        \end{align}
    
        The proof follows the proof of Proposition 4.1 and Theorem 4.5 of \citet{sloman_bayesian_2024}.
        It depends on the following definitions:

        \newcommand{\neighborhood}[1]{N_{\epsilon} \left( #1 \right)}
        \newcommand{\restrictPrior}[1]{P^{#1}_{\paramRV}}
        \begin{definition}[$\epsilon$-neighborhood of $\param$ $\neighborhood{\param}$ (Definition 4.2 of \citealt{sloman_bayesian_2024})]
            $\neighborhood{\param} \equiv \{ \param^{\prime} \in \paramvalues ~ \vert ~ d(\param, \param^{\prime}) < \epsilon \}$, where $d$ is a suitable distance measure, is the $\epsilon$-neighborhood of $\param$.
        \end{definition}

        \begin{definition}[$\restrictPrior{\mathscr{A}}$ (modification of Definition 4.3 of \citealt{sloman_bayesian_2024})]
            $\restrictPrior{\mathscr{A}}$ refers to the distribution of $\paramRV$ obtained by restricting the support of the learner's prior to the set $\mathscr{A}$, under which
            $$p^{\mathscr{A}}(\param) \equiv \frac{p(\param)}{\int_{\mathscr{A}} p(\param) ~ d\param}$$ for any $\param \in \mathscr{A}$.
        \end{definition}
    
        \begin{assumption}[Smoothness in parameter space (Assumption 4.4 of \citet{sloman_bayesian_2024})]\label{as:smoothness}
            There exists some $\epsilon > 0$ such that
            \begin{align}
                &\E{\sourcedata \sim P_{\truedataRV}}{\log{\left( \E{\param \sim P_{\paramRV}}{L(\sourcedata, \param)} \right)}} \nonumber \\ 
                &\geq \E{\sourcedata \sim P_{\truedataRV}}{\left( \int_{N_{\epsilon} \left( \targetconcept \right)} p(\param) ~ d\param \right) \log{\left( L(\sourcedata, \targetconcept) \right)} + \left( \int_{\paramvalues \backslash N_{\epsilon} \left( \targetconcept \right)} p(\param) ~ d\param \right) \log{\left( \E{\param \sim P^{\paramvalues \backslash N_{\epsilon} \left( \targetconcept \right)}_{\paramRV}}{L(\sourcedata, \param)} \right)}} \nonumber
            \end{align}
            where $\left( \int_{N_{\epsilon} \left( \targetconcept \right)} p(\param) ~ d\param \right)$ and $\left( \int_{\paramvalues \backslash N_{\epsilon} \left( \targetconcept \right)} p(\param) ~ d\param \right)$ are the probability that a value $\param$ is inside and outside the $\epsilon$-neighborhood of $\targetconcept$, respectively.
        \end{assumption}
        \begin{remark}
            \Cref{as:smoothness} holds when $\paramRV$ is a discrete random variable (in which case the $\epsilon$-neighborhood of $\targetconcept$ can be defined as $\{ \targetconcept \}$ and to exclude all other parameter values).
            When $\paramRV$ is a continuous random variable, \Cref{as:smoothness} is essentially a smoothness condition: For likelihoods that are sufficiently smooth around $\targetconcept$, we can expect it to hold for $\epsilon \rightarrow 0$.
            To see this, notice that Jensen's inequality implies that
            \begin{align}
                &\E{\sourcedata \sim P_{\truedataRV}}{\log{\left( \E{\param \sim P_{\paramRV}}{L(\sourcedata, \param)} \right)}} \nonumber \\
                &\geq \E{\sourcedata \sim P_{\truedataRV}}{\left( \int_{N_{\epsilon} \left( \targetconcept \right)} p(\param) ~ d\param \right) \log{\left( \E{\param \sim P^{N_{\epsilon} \left( \targetconcept \right)}_{\paramRV}}{L(\sourcedata, \param)} \right)} + \left( \int_{\paramvalues \backslash N_{\epsilon} \left( \targetconcept \right)} p(\param) ~ d\param \right) \log{\left( \E{\param \sim P^{\paramvalues \backslash N_{\epsilon} \left( \targetconcept \right)}_{\paramRV}}{L(\sourcedata, \param)} \right)}}. \nonumber
            \end{align}
            \Cref{as:smoothness} holds when $\E{\param \sim P^{N_{\epsilon} \left( \targetconcept \right)}_{\paramRV}}{p(\sourcedata \vert \param)} \approx p(\sourcedata \vert \targetconcept)$ and the approximation is tight enough that it does not close the Jensen gap.
        \end{remark}

        Taking $P_{\sourcedataRV \vert \param \in \paramvalues \backslash \neighborhood{\targetconcept}}$ to be the source data distribution conditioned on the event that the shared parameter is not in the $\epsilon$-neighborhood of $\targetconcept$, we obtain
        \begin{align}
            \cIG\left( \targetconcept \right) &= \E{\sourcedata \sim P_{\truedataRV}}{\log{\left( L\left( \sourcedata, \targetconcept \right) \right)} - \log{\left( \E{\param \sim P_{\paramRV}}{L\left( \sourcedata, \param \right)} \right)}} \nonumber \\
            &\leq \mathbb{E}_{\sourcedata \sim P_{\truedataRV}} \left[ \log{\left( L(\sourcedata, \targetconcept) \right)} - \left( \int_{N_{\epsilon} \left( \targetconcept \right)} p(\param) ~ d\param \right) \log{\left( L(\sourcedata, \targetconcept) \right)} - \right. \nonumber \\
            &\left. \phantom{\mathbb{E}_{\sourcedata \sim P_{\sourcedataRV \vert \targetconcept, \truetsparam}}} \left( \int_{\paramvalues \backslash N_{\epsilon} \left( \targetconcept \right)} p(\param) ~ d\param \right) \log{\left( \E{\param \sim P^{\paramvalues \backslash N_{\epsilon} \left( \targetconcept \right)}_{\paramRV}}{L(\sourcedata, \param)} \right)} \right] ~~~~~~~~~~~~~~~~~~~~~~~~~~~~~~~~~~~~~~~~~~ \text{(\Cref{as:smoothness})} \nonumber \\
            &= \E{\sourcedata \sim P_{\truedataRV}}{\left( \int_{\paramvalues \backslash N_{\epsilon} \left( \targetconcept \right)} p(\param) ~ d\param \right) \left( \log{\left( L(\sourcedata, \targetconcept) \right)} - \log{\left( \E{\param \sim P^{\paramvalues \backslash N_{\epsilon} \left( \targetconcept \right)}_{\paramRV}}{L(\sourcedata, \param)} \right)} \right)} \nonumber \\
            &= \left( \int_{\paramvalues \backslash N_{\epsilon} \left( \targetconcept \right)} p(\param) ~ d\param \right) \left( \crossent{P_{\truedataRV}}{P_{\sourcedataRV \vert \param \in \paramvalues \backslash \neighborhood{\targetconcept}}} - \crossent{P_{\truedataRV}}{P_{\sourcedataRV \vert \targetconcept}} \right) \nonumber \\
            &= \left( \int_{\paramvalues \backslash N_{\epsilon} \left( \targetconcept \right)} p(\param) ~ d\param \right) \left( \ent{P_{\truedataRV}} + \kld{P_{\truedataRV}}{P_{\sourcedataRV \vert \param \in \paramvalues \backslash \neighborhood{\targetconcept}}} - \ent{P_{\truedataRV}} - \kld{P_{\truedataRV}}{P_{\sourcedataRV \vert \targetconcept}} \right) \nonumber \\
            &= \left( \int_{\paramvalues \backslash N_{\epsilon} \left( \targetconcept \right)} p(\param) ~ d\param \right) \left( \kld{P_{\truedataRV}}{P_{\sourcedataRV \vert \param \in \paramvalues \backslash \neighborhood{\targetconcept}}} - \kld{P_{\truedataRV}}{P_{\sourcedataRV \vert \targetconcept}} \right)
        \end{align}
        as stated in the theorem for $A = \left( \int_{\paramvalues \backslash N_{\epsilon} \left( \targetconcept \right)} p(\param) ~ d\param \right)$ and $B = \kld{P_{\truedataRV}}{P_{\sourcedataRV \vert \param \in \paramvalues \backslash \neighborhood{\targetconcept}}}$.

    \subsection{Proof of Theorem 4.4}\label{ap:reweighted-ig}
       The r-weighted information gain (\Cref{def:pir-tig}) can be written as:
        \begin{align}
            \rIG\left( \targetconcept \right) &= \E{\sourcedata,\auxInfo \sim P_{\truedataRV,\auxInfoRV}}{\log{\left( \frac{p^{\mathcal{R}}(\targetconcept \vert \sourcedata, \auxInfo)}{p(\targetconcept)} \right)}} \nonumber \\
            &= \E{\sourcedata,\auxInfo \sim P_{\truedataRV,\auxInfoRV}}{\log{\left( \frac{\frac{\E{\tsparam_{n+1} \sim P_{\targettsparamRV \vert \auxInfo}}{\robustL\left( \sourcedata, \targetconcept, \tsparam = \tsparam_{n+1} \right)} ~ p\left( \targetconcept \right)}{\E{\param, \tsparam_{n+1}^{\prime} \sim P_{\paramRV, \targettsparamRV}}{\robustL\left( \sourcedata, \param, \tsparam = \tsparam_{n+1}^{\prime} \right)}}}{p\left( \targetconcept \right)} \right)}} \nonumber \\
            &= \E{\sourcedata,\auxInfo \sim P_{\truedataRV,\auxInfoRV}}{\log{\left( \frac{\E{\tsparam_{n+1} \sim P_{\targettsparamRV \vert \auxInfo}}{\robustL\left( \sourcedata, \targetconcept, \tsparam = \tsparam_{n+1} \right)}}{\E{\param,\tsparam_{n+1}^{\prime} \sim P_{\paramRV,\targettsparamRV}}{\robustL\left( \sourcedata, \param, \tsparam = \tsparam_{n+1}^{\prime} \right)}} \right)}}
        \end{align}

        \begin{remark}
            Notice that $\rIG\left( \targetconcept \right)$ is defined as an expectation over $P_{\auxInfoRV}$ as well as $P_{\truedataRV}$.
            This, and all other quantities in our analysis which include expectations over $\auxInfoRV$, can be interpreted as marginalizing across the learner's subjective uncertainty about the proxy information they will receive.
            We could have defined $\rIG\left( \targetconcept \right)$ as an expectation across a ``true'' distribution of proxy information, with a corresponding interpretation as the extent to which the learner can expect to gain information upon encountering a given distribution generating \textnormal{both} source data and proxy information.
            Although such an extension of the current analysis would in some sense be technically more complete, we opt to simplify our analysis and define the expectation over proxy information with respect to the learner's subjective uncertainty.
            Both the learner using a classic likelihood and the learner using an r-weighted likelihood use the same prior over $\auxInfoRV$ in estimation of $\targettsparam$, and so the incorrectness of the prior over proxy information is less important than the incorrectness of the prior over source task parameters in understanding the relative advantage of r-weighting.
        \end{remark}
    
        The proof of \Cref{prop:ig-reweighted} uses the following lemma:
        \begin{lemma}\label{lem:jensen-gap}
            Define $\mathcal{J}\left( \log{}; P_{\targettsparamRV \vert \auxInfo} \right) \equiv \log{\left( \E{\tsparam_{n+1} \sim P_{\targettsparamRV \vert \auxInfo}}{\robustL\left( \sourcedata, \targetconcept, \tsparam = \tsparam_{n+1} \right)} \right)} - \E{\tsparam_{n+1} \sim P_{\targettsparamRV \vert \auxInfo}}{\log{\left( \robustL\left( \sourcedata, \targetconcept, \tsparam = \tsparam_{n+1} \right) \right)}}$ and $\mathcal{J}\left( \log{}; P_{\targettsparamRV} \right) \equiv \log{\left( \E{\param,\tsparam_{n+1} \sim P_{\paramRV,\targettsparamRV}}{\robustL\left( \sourcedata, \param, \tsparam = \tsparam_{n+1} \right)} \right)} - \E{\tsparam_{n+1} \sim P_{\targettsparamRV}}{\log{\left( \E{\param \sim P_{\paramRV}}{\robustL\left( \sourcedata, \param, \tsparam = \tsparam_{n+1} \right)} \right)}}$.
            Under \Cref{as:jensen-gap} and \Cref{as:convergence} (stated formally in the proof of the lemma), $\E{\sourcedata \sim P_{\truedataRV}}{\mathcal{J}\left( \log{}; P_{\targettsparamRV} \right)} \geq \E{\sourcedata, \auxInfo \sim P_{\truedataRV,\auxInfoRV}}{\mathcal{J}\left( \log{}; P_{\targettsparamRV \vert \auxInfo} \right)}$.
        \end{lemma}
        \begin{proof}[Proof of \Cref{lem:jensen-gap}]
            The lemma leverages a result known as H\"{o}lder's defect \citep{steele_cauchy-schwarz_2004,becker_variance_2012}:
            \begin{theorem}[H\"{o}lder's defect (restated from \citealt{steele_cauchy-schwarz_2004}\footnote{\citet{steele_cauchy-schwarz_2004} states the result in terms of discrete sums; we here modified the statement of the result so it can be interpreted for continuous random variables.})
            ]
                If $f ~ : ~ \left[ a, b \right] \rightarrow \mathbb{R}$ is twice differentiable and if we have the bounds
                $$0 \leq m \leq f^{\prime \prime}(x) \leq M ~ \text{for all} ~ x \in \left[ a, b \right],$$
                then for a distribution $P$ over $\left[ a, b \right]$, there exists a real value $\mu \in \left[ m, M \right]$ for which one has the formula
                $$\underbrace{\E{x \sim P}{f(x)} - f\left( \E{x \sim P}{x} \right)}_{\mathcal{J}\left( -f; P \right)} = \frac{1}{2} \mu \mathrm{Var}_{x \sim P}\left[ x \right]$$
                for $\mathrm{Var}_{x \sim P}\left[ x \right] \equiv \E{x \sim P}{\left( x - \E{x \sim P}{x} \right)^2}$.
            \end{theorem}

            Our goal is to use H\"{o}lder's defect to relate $\mathcal{J}\left( \log{}; P_{\targettsparamRV \vert \auxInfo} \right)$ and $\mathcal{J}\left( \log{}; P_{\targettsparamRV} \right)$ to $\mathrm{Var}_{\tsparam_{n+1} \sim P_{\tsparamRV_{n+1} \vert \auxInfo}} \left[ \robustL\left( \sourcedata, \targetconcept, \tsparam = \tsparam_{n+1} \right) \right]$ and $\mathrm{Var}_{\tsparam_{n+1} \sim P_{\tsparamRV_{n+1}}} \left[ \E{\param \sim P_{\paramRV}}{\robustL\left( \sourcedata, \param, \tsparam = \tsparam_{n+1} \right)} \right]$, respectively.
            We first verify that the conditions required for H\"{o}lder's defect formula to apply are met.
            For both applications of the result, $f$ is the negative of the log function.
            In application to $\mathcal{J}\left( \log{}; P_{\targettsparamRV \vert \auxInfo} \right)$, $f$ takes as input values of $\robustL\left( \sourcedata, \targetconcept, \tsparam = \tsparam_{n+1} \right)$.
            In application to $\mathcal{J}\left( \log{}; P_{\targettsparamRV} \right)$, $f$ takes as input values of $\E{\param \sim P_{\paramRV}}{\robustL\left( \sourcedata, \param, \tsparam = \tsparam_{n+1} \right)}$.
            \begin{itemize}
                \item $f ~ : ~ \left[ a, b \right] \rightarrow \mathbb{R}$: \Cref{as:jensen-gap} ensures that inputs in both cases are bounded from both below and above.
                \item $f$ is twice differentiable: The second derivative of $f$ evaluated at $x$ is $f^{\prime \prime}(x) = \frac{1}{x^{2}}$.
                \item $0 \leq m \leq f^{\prime \prime}(x) \leq M$: \Cref{as:jensen-gap} ensures this for $m = \frac{1}{b^2}$ and $M = \frac{1}{a^2}$.
            \end{itemize}

            H\"{o}lder's defect then implies the following:
            \begin{align}\label{eq:jensen-psi_z}
                \mathcal{J}\left( \log{}; P_{\targettsparamRV \vert \auxInfo} \right) &= \frac{1}{2} \mu_1\left( \sourcedata, \auxInfo \right) \mathrm{Var}_{\tsparam_{n+1} \sim P_{\tsparamRV_{n+1} \vert \auxInfo}} \left[ \robustL\left( \sourcedata, \targetconcept, \tsparam = \tsparam_{n+1} \right) \right]
            \end{align}
            for a scalar $\mu_1$ that depends on $\sourcedata$ and $\auxInfo$, and
            \begin{align}\label{eq:jensen-psi}
                \mathcal{J}\left( \log{}; P_{\targettsparamRV} \right) &= \frac{1}{2} \mu_2\left( \sourcedata \right) \mathrm{Var}_{\tsparam_{n+1} \sim P_{\tsparamRV_{n+1}}} \left[ \E{\param \sim P_{\paramRV}}{\robustL\left( \sourcedata, \param, \tsparam = \tsparam_{n+1} \right)} \right]
            \end{align}
        for a scalar $\mu_2$ that depends on $\sourcedata$.

        We can now formally state \Cref{as:convergence}:
        \begin{assumption}[Sufficiently informative proxy]\label{as:convergence}
            The proxy is sufficiently informative in the sense that the following condition holds on the relative variances of $\targettsparamRV \vert \auxInfo$ and $\targettsparamRV$:
            \begin{align}
                &\E{\sourcedata \sim P_{\truedataRV}}{\mu_2\left( \sourcedata \right) \mathrm{Var}_{\tsparam_{n+1} \sim P_{\tsparamRV_{n+1}}} \left[ \E{\param \sim P_{\paramRV}}{\robustL\left( \sourcedata, \param, \tsparam = \tsparam_{n+1} \right)} \right]} \geq \nonumber \\
                &\E{\sourcedata, \auxInfo \sim P_{\truedataRV, \auxInfoRV}}{\mu_1\left( \sourcedata, \auxInfo \right) \mathrm{Var}_{\tsparam_{n+1} \sim P_{\tsparamRV_{n+1} \vert \auxInfo}} \left[ \robustL\left( \sourcedata, \targetconcept, \tsparam = \tsparam_{n+1} \right) \right]}. \nonumber
            \end{align}
        \end{assumption}

        Direct substitution of the condition in \Cref{as:convergence} into \Cref{eq:jensen-psi_z,eq:jensen-psi} completes the proof.
    \end{proof}

    In addition to \Cref{lem:jensen-gap}, the proof of \Cref{prop:ig-reweighted} uses the following assumption, which is a variant of \Cref{as:smoothness} for the r-weighted case:
    \begin{assumption}[Smoothness in parameter space]\label{as:smoothness-reweighted}
        There exists some $\epsilon > 0$ such that
        \begin{align}
            &\E{\sourcedata, \tsparam_{n+1} \sim P_{\truedataRV, \targettsparamRV}}{\log{\left( \E{\param \sim P_{\paramRV}}{\robustL(\sourcedata, \param, \tsparam = \tsparam_{n+1})} \right)}} \geq \nonumber \\
            &\mathbb{E}_{\sourcedata, \tsparam_{n+1} \sim P_{\truedataRV, \targettsparamRV}} \left[ \left( \int_{N_{\epsilon} \left( \targetconcept \right)} p(\param) ~ d\param \right) \log{\left( \robustL(\sourcedata, \targetconcept, \tsparam = \tsparam_{n+1}) \right)} \right. \nonumber \\
            &\phantom{\mathbb{E}_{\sourcedata, \tsparam_{n+1} \sim P_{\truedataRV, \targettsparamRV}}} + \left. \left( \int_{\paramvalues \backslash N_{\epsilon} \left( \targetconcept \right)} p(\param) ~ d\param \right) \log{\left( \E{\param \sim P^{\paramvalues \backslash N_{\epsilon} \left( \targetconcept \right)}_{\paramRV}}{\robustL(\sourcedata, \param, \tsparam = \tsparam_{n+1})} \right)} \right] \nonumber
        \end{align}
        where $\left( \int_{N_{\epsilon} \left( \targetconcept \right)} p(\param) ~ d\param \right)$ and $\left( \int_{\paramvalues \backslash N_{\epsilon} \left( \targetconcept \right)} p(\param) ~ d\param \right)$ are the probability that a value $\param$ is inside and outside the $\epsilon$-neighborhood of $\targetconcept$, respectively.
    \end{assumption}
    \begin{remark}
        In addition to the smoothness condition on the likelihood imposed by \Cref{as:smoothness}, \Cref{as:smoothness-reweighted} additionally imposes what is essentially a ceiling on the outputs of the relevance function.
        Weights $< 1$ ``flatten'', or smooth out, the likelihood function; weights $> 1$ ``sharpen'' it, and may cause violation of \Cref{as:smoothness-reweighted} even in cases where \Cref{as:smoothness} is met.
        The relevance functions used in our examples (\Cref{sec:experiments}) output weights $\leq 1$.
    \end{remark}

    \allowdisplaybreaks
    We obtain
    \begin{align}
        \rIG\left( \targetconcept \right) = &\E{\sourcedata, \auxInfo \sim P_{\truedataRV, \auxInfoRV}}{\log{\left( \frac{\E{\tsparam_{n+1} \sim P_{\tsparamRV_{n+1} \vert \auxInfo}}{\robustL\left( \sourcedata, \targetconcept, \tsparam = \tsparam_{n+1} \right)}}{\E{\param, \tsparam^{\prime}_{n+1} \sim P_{\paramRV, \tsparamRV_{n+1}}}{\robustL \left( \sourcedata, \param, \tsparam_{n+1}^{\prime} \right)}} \right)}} \nonumber \\
        \text{(\Cref{lem:jensen-gap})} ~~~~~~ \leq &\E{\sourcedata, \auxInfo \sim P_{\truedataRV, \auxInfoRV}}{\E{\tsparam_{n+1} \sim P_{\tsparamRV_{n+1} \vert \auxInfo}}{\log{\left( \robustL\left( \sourcedata, \targetconcept, \tsparam = \tsparam_{n+1} \right) \right)}}} \nonumber \\
        &- \E{\sourcedata \sim P_{\truedataRV}}{\E{\tsparam_{n+1} \sim P_{\tsparamRV_{n+1}}}{\log{\left( \E{\param \sim P_{\paramRV}}{\robustL\left( \sourcedata, \param, \tsparam = \tsparam_{n+1} \right)} \right)}}} \nonumber \\
        = &\E{\sourcedata \sim P_{\truedataRV}}{\E{\tsparam_{n+1} \sim P_{\tsparamRV_{n+1}}}{\log{\left( \robustL\left( \sourcedata, \targetconcept, \tsparam = \tsparam_{n+1} \right) \right)}}} \nonumber \\
        &- \E{\sourcedata \sim P_{\truedataRV}}{\E{\tsparam_{n+1} \sim P_{\tsparamRV_{n+1}}}{\log{\left( \E{\param \sim P_{\paramRV}}{\robustL\left( \sourcedata, \param, \tsparam = \tsparam_{n+1} \right)} \right)}}} \nonumber \\
        = &\E{\sourcedata, \tsparam_{n+1} \sim P_{\truedataRV, \tsparamRV_{n+1}}}{\log{\left( \robustL\left( \sourcedata, \targetconcept, \tsparam = \tsparam_{n+1} \right) \right)} - \log{\left( \E{\param \sim P_{\paramRV}}{\robustL\left( \sourcedata, \param, \tsparam = \tsparam_{n+1} \right)} \right)}} \nonumber \\
        \text{(\Cref{as:smoothness-reweighted})} \leq &\mathbb{E}_{\sourcedata, \tsparam_{n+1} \sim P_{\truedataRV, \tsparamRV_{n+1}}}\left[\log{\left( \robustL\left( \sourcedata, \targetconcept, \tsparam = \tsparam_{n+1} \right) \right)} - \left( \int_{N_{\epsilon} \left( \targetconcept \right)} p(\param) ~ d\param \right) \log{\left( \robustL(\sourcedata, \targetconcept, \tsparam = \tsparam_{n+1}) \right)} - \right. \nonumber \\
        &\phantom{\mathbb{E}_{\sourcedata,\tsparam_{n+1} \sim P_{\truedataRV,\tsparamRV_{n+1}}}\left[\right. ~~~~~} \left. \left( \int_{\paramvalues \backslash N_{\epsilon} \left( \targetconcept \right)} p(\param) ~ d\param \right) \log{\left( \E{\param \sim P^{\paramvalues \backslash N_{\epsilon} \left( \targetconcept \right)}_{\paramRV}}{\robustL(\sourcedata, \param, \tsparam = \tsparam_{n+1})} \right)} \right] \nonumber \\
        = &\mathbb{E}_{\tsparam_{n+1} \sim P_{\tsparamRV_{n+1}}} \left[ \mathbb{E}_{\sourcedata \sim P_{\truedataRV}} \left[ \left( \int_{\paramvalues \backslash N_{\epsilon} \left( \targetconcept \right)} p(\param) ~ d\param \right) \left( \log{\left( \robustL\left( \sourcedata, \targetconcept, \tsparam = \tsparam_{n+1} \right) \right)} \right. \right. \right. \nonumber \\
        &\left. \left. \left. \phantom{\mathbb{E}_{\tsparam_{n+1} \sim P_{\tsparamRV_{n+1}}} \left[ \mathbb{E}_{\sourcedata \sim P_{\truedataRV}} \left[ \left( \int_{\paramvalues \backslash N_{\epsilon} \left( \targetconcept \right)} p(\param) ~ d\param \right) \right. \right.} - \log{\left( \E{\param \sim P^{\paramvalues \backslash N_{\epsilon} \left( \targetconcept \right)}_{\paramRV}}{\robustL(\sourcedata, \param, \tsparam = \tsparam_{n+1})} \right)} \right) \right] \right] \nonumber \\
        = &\left( \int_{\paramvalues \backslash N_{\epsilon} \left( \targetconcept \right)} p(\param) ~ d\param \right) \mathbb{E}_{\tsparam_{n+1} \sim P_{\tsparamRV_{n+1}}} \left[ \crossent{P_{\truedataRV}}{P_{\sourcedataRV^{\mathcal{R}(\tsparam_{n+1})} \vert \param \in \paramvalues \backslash \neighborhood{\targetconcept}, \tsparam = \tsparam_{n+1}}} \right. \nonumber \\
        &\phantom{\left( \int_{\paramvalues \backslash N_{\epsilon} \left( \targetconcept \right)} p(\param) ~ d\param \right) \mathbb{E}_{\tsparam_{n+1} \sim P_{\tsparamRV_{n+1}}}} - \left. \crossent{P_{\truedataRV}}{P_{\sourcedataRV^{\mathcal{R}(\tsparam_{n+1})} \vert \targetconcept, \tsparam = \tsparam_{n+1}}} \right] \nonumber \\
        = &\left( \int_{\paramvalues \backslash N_{\epsilon} \left( \targetconcept \right)} p(\param) ~ d\param \right) \mathbb{E}_{\tsparam_{n+1} \sim P_{\tsparamRV_{n+1}}} \left[ \ent{P_{\truedataRV}} + \kld{P_{\truedataRV}}{P_{\sourcedataRV^{\mathcal{R}(\tsparam_{n+1})} \vert \param \in \paramvalues \backslash \neighborhood{\targetconcept}, \tsparam = \tsparam_{n+1}}} \right. \nonumber \\
        &\phantom{\left( \int_{\paramvalues \backslash N_{\epsilon} \left( \targetconcept \right)} p(\param) ~ d\param \right) \mathbb{E}_{\tsparam_{n+1} \sim P_{\tsparamRV_{n+1}}} \left[ \right.} \left.- \ent{P_{\truedataRV}} - \kld{P_{\truedataRV}}{P_{\sourcedataRV^{\mathcal{R}(\tsparam_{n+1})} \vert \targetconcept, \tsparam = \tsparam_{n+1}}} \right] \nonumber \\
        = &\left( \int_{\paramvalues \backslash N_{\epsilon} \left( \targetconcept \right)} p(\param) ~ d\param \right) \mathbb{E}_{\tsparam_{n+1} \sim P_{\tsparamRV_{n+1}}} \left[ \kld{P_{\truedataRV}}{P_{\sourcedataRV^{\mathcal{R}(\tsparam_{n+1})} \vert \param \in \paramvalues \backslash \neighborhood{\targetconcept}, \tsparam = \tsparam_{n+1}}} \right. \nonumber \\
        &\phantom{\left( \int_{\paramvalues \backslash N_{\epsilon} \left( \targetconcept \right)} p(\param) ~ d\param \right) \mathbb{E}_{\tsparam_{n+1} \sim P_{\tsparamRV_{n+1}}}} \left. - \kld{P_{\truedataRV}}{P_{\sourcedataRV^{\mathcal{R}(\tsparam_{n+1})} \vert \targetconcept, \tsparam = \tsparam_{n+1}}} \right]
    \end{align}
    as stated in the theorem for $A = \left( \int_{\paramvalues \backslash N_{\epsilon} \left( \targetconcept \right)} p(\param) ~ d\param \right)$ and $C = \E{\tsparam_{n+1} \sim P_{\tsparamRV_{n+1}}}{\kld{P_{\truedataRV}}{P_{\sourcedataRV^{\mathcal{R}(\tsparam_{n+1})} \vert \param \in \paramvalues \backslash \neighborhood{\targetconcept}, \tsparam = \tsparam_{n+1}}}}$.

    \subsection{Proof of Proposition 5.5}\label{ap:reweighted-rho}
        The proof depends on the following definition:
        \begin{definition}[Fidelity of the relevance function $\rho^{\mathcal{R}}$]\label{def:rho}
            $\rho^{\mathcal{R}}$ is a measure of the fidelity of the relevance function.
            More specifically, it is:
            \begin{align}
                \rho^{\mathcal{R}} \equiv \mathbb{E}_{\sourcedata, \tsparam_{n+1} \sim P_{\truedataRV,\tsparamRV_{n+1}}}
                &\left[ \frac{1}{n} \sum_{i=1}^n \left( \csim{i}{\tsparam_{n+1}} - \frac{1}{n} \sum_{i=1}^n \csim{i}{\tsparam_{n+1}} \right) \right. \nonumber \\
                &\phantom{\frac{1}{n} \sum_{i=1}^n} \left.
                \left( \log{\left( p\left( \sourcedata_i \vert \targetconcept, \tsparam_i = \tsparam_{n+1} \right) \right)} -  \frac{1}{n} \sum_{i=1}^n \log{\left( p\left( \sourcedata_i \vert \targetconcept, \tsparam_i = \tsparam_{n+1} \right) \right)} \right)
                \right], \nonumber
            \end{align}
            i.e., is the covariance of $\csim{i}{\tsparam_{n+1}}$ and $\log{\left( p\left( \sourcedata_i \vert \targetconcept, \tsparam_i = \tsparam_{n+1} \right) \right)}$ with respect to a uniform distribution over the source data, in expectation over $P_{\truedataRV,\tsparamRV_{n+1}}$.
        \end{definition}

        $\rDelta$ can be rewritten as
        \begin{align}\label{eq:cor-reweighting}
            \rDelta &= \E{\tsparam_{n+1} \sim P_{\tsparamRV_{n+1}}}{\E{\sourcedata \sim P_{\truedataRV}}{\log{\left( \frac{L\left( \sourcedata, \targetconcept, \truetsparam \right)}{\robustL\left( \sourcedata, \targetconcept, \tsparam = \tsparam_{n+1} \right)} \right)}}} \nonumber \\
            &= - \E{\sourcedata, \tsparam_{n+1} \sim P_{\truedataRV,\tsparamRV_{n+1}}}{\log{\left( \robustL\left( \sourcedata, \targetconcept, \tsparam = \tsparam_{n+1} \right) \right)}} - \ent{P_{\truedataRV}} \nonumber \\
            &= - \E{\sourcedata,\tsparam_{n+1} \sim P_{\truedataRV,\tsparamRV_{n+1}}}{\sum_{i=1}^n \csim{i}{\tsparam_{n+1}} \log{\left( p\left( \sourcedata_i \vert \targetconcept, \tsparam_i = \tsparam_{n+1} \right) \right)}} - \ent{P_{\truedataRV}} \nonumber \\
            &= \E{\sourcedata,\tsparam_{n+1} \sim P_{\truedataRV,\tsparamRV_{n+1}}}{\left( \sum_{i=1}^n \csim{i}{\tsparam_{n+1}} \right) \left( - \sum_{i=1}^n \log{\left( p\left( \sourcedata_i \vert \targetconcept, \tsparam_i = \tsparam_{n+1} \right) \right)} \right)} - n \rho^{\mathcal{R}} - \ent{P_{\truedataRV}} \nonumber \\
            &= \E{\sourcedata,\tsparam_{n+1} \sim P_{\truedataRV,\tsparamRV_{n+1}}}{\left( \sum_{i=1}^n \csim{i}{\tsparam_{n+1}} \right) \left( - \log{\left( p\left( \sourcedata \vert \targetconcept, \tsparam = \tsparam_{n+1} \right) \right)} \right)} - n \rho^{\mathcal{R}} - \ent{P_{\truedataRV}} ~~~~~~~~~ \text{(\Cref{as:theta-psi-ind})} \nonumber
        \end{align}
        as stated in the proposition for $D = - \ent{P_{\truedataRV}}$.

        \begin{remark}
            As discussed in \Cref{sec:step2}, in practice the learner can often specify a sufficiently high-fidelity relevance function even in the absence of knowledge of $\targetconcept$, i.e., a relevance function for which $\rho^{\mathcal{R}}$ is sufficiently large.
            An example relevance function is given in \Cref{eq:iterR}.
            However, this relevance function is not guaranteed to positively correlate with $p\left( \sourcedata_i \vert \targetconcept, \tsparam_i = \tsparam_{n+1} \right)$.
            If the learner is particularly unlucky, this relevance function could have a negative corresponding value of $\rho^{\mathcal{R}}$, i.e., increase the relevance of source data points \textnormal{least} likely under a particular pseudo-intervention.
            This may occur if $\param$ and $\tsparam$ interact such that the direction of the gradient of predictions with respect to $\tsparam$ depends on $\param$.
            For example, consider a case in which for all except very few values of $\param$, outcomes increase as a function of $\tsparam$.
            In the context of our motivating example of treatment effect estimation, this might correspond to a situation where hospital quality generally increases the relative effectiveness of a treatment, unless the treatment effect is very extreme (in which case hospital quality has a larger impact on the effect of a placebo).
            If the true treatment effect is in fact very extreme, the relevance function shown in \Cref{eq:iterR} would likely negatively correlate with $p\left( \sourcedata_i \vert \targetconcept, \tsparam_i = \tsparam_{n+1} \right)$.
        \end{remark}

\section{DETAILS OF EXAMPLES}\label{ap:experiments}
    We here report the details of the examples described in \Cref{sec:experiments}.
    \Cref{ap:linreg} gives details of the linear regression example, \Cref{ap:smoking} gives details of the example predicting smoking behavior, and \Cref{ap:gp} gives details of the GP regression example and results showing the relative performance of \algAbbrev{} as a function of the values of each of several simulation parameters.

    \subsection{Linear regression}\label{ap:linreg}
        All simulations were run using only a CPU.
        In all simulations, the value of the shared parameter $\targetconcept = -1$.
        The prior $P_{\paramRV,\tsparamRV_i} = \mathcal{N}\left( [ 0, 0 ]^{\top}, \mathrm{diag}\left( [ 1, 1 ] \right) \right)$ for all $i \in 1:n+1$.
        
        To generate source data, we first specified a particular level of multicollinearity $\rho$.
        A higher degree of multicollinearity makes $\targetconcept$ and $\targettsparam$ harder to separately identify, so we interpret this as a higher risk of negative transfer.
        We varied $\rho$ among 0 (no multicollinearity), 1 (mild multicollinearity), and 2 (extreme multicollinearity).
        
        For a given value $\rho$ we sampled values $\sourcex^{\prime} \sim \mathcal{N}\left( \rho, .25 \right)$, and then constructed values $\sourcex_{(\cdot,1)} \sim \mathcal{N}\left( \sourcex^{\prime}, .25 \right)$ and values $\sourcex_{(\cdot,2)} \sim \mathcal{N}\left( -\frac{\rho^2}{\sourcex^{\prime}}, .25 \right)$.
        We created 100 such data points.
        Twenty-five of these data points were used to create proxy information (i.e., used to generate values $\auxInfo_i$ as described below), and 75 were used as outcome information on the basis of which to estimate the shared parameter.

        In the simulations shown in \Cref{fig:betabinom}, all proxy values are uncontaminated.
        In the simulations shown in \Cref{fig:linear-contamination}, $\rho = 2$ always, i.e., all simulations are run in the presence of extreme multicollinearity.

        \paragraph{Relevance function.}
            We first computed the relevances as $\csim{i}{\tsparam_{n+1}} \propto p\left( \sourcedata_i \vert \tsparam_i = \tsparam_{n+1} \right) = \E{\param \sim P_{\paramRV}}{p\left( \sourcedata_i \vert \param, \tsparam_i = \tsparam_{n+1} \right)}$ where the constant of proportionality was the probability a distribution with the same variance would assign to its mode.
            Using the calculated relevances, we computed $P^{\mathcal{R}}_{\paramRV, \targettsparamRV \vert \sourcedata, \auxInfo}$.
            We then defined $\widehat{P}^{\mathcal{R}}_{\paramRV}$ as a Gaussian approximation to samples from the r-weighted posterior $P^{\mathcal{R}}_{\paramRV \vert \sourcedata, \auxInfo}$, recomputed each $\csim{i}{\tsparam_{n+1}} \propto \E{\param \sim \widehat{P}^{\mathcal{R}}_{\paramRV}}{p\left( \sourcedata_i \vert \param, \tsparam_i = \tsparam_{n+1} \right)}$, and recomputed the r-weighted posterior.
            In each simulation, we repeated this three times before ultimately defining the relevance function as an expectation across the distribution $\widehat{P}^{\mathcal{R}}_{\paramRV}$ obtained at the final iteration.

        \paragraph{Proxy information.}
            When $q\%$ of proxy values are contaminated, $1 - q\%$ of proxy values are generated as $\auxInfo \sim \mathrm{Binomial}\left( 7, \tilde{p}(\sourcedata^{\prime} \vert \tsparam^{\prime} = \targettsparam) \right)$ where $\sourcedata^{\prime}$ are observations used to prompt the synthetic expert for feedback, $\tsparam^{\prime}$ are the corresponding task parameters, and $\tilde{p}$ indicates that the probability has been normalized to not exceed 1.
            The remaining $q\%$ of proxy values are generated as $\auxInfo \sim \mathrm{Binomial}\left( 7, 1 - \tilde{p}(\sourcedata^{\prime} \vert \tsparam^{\prime} = \targettsparam) \right)$.

    \subsection{Predicting smoking behavior}\label{ap:smoking}
        All computations were run using only a CPU.
        The prior for all effects in both the classic and r-weighted models was $\mathcal{N}\left( 0, 3 \right)$.

        The classic Bayesian learner estimated the fixed effects model
        $$\sourcey_i \vert \sourcex_i, \param, \tsparam \sim \mathrm{Binomial}\left( \mathrm{sigmoid}\left( \param \sourcex_{i,(1:4)}^{\top} + \tsparam \sourcex_{i,(5:28)}^{\top} \right), N_i \right)$$
        where $\sourcex_{i,(1:4)}$ are indicators of the treatment received, $\sourcex_{i,(5:28)}$ are study indicators, and $N_i$ is the number of patients who received the indicated treatment in the indicated study.
        The classic learner's estimate of the study indicator for the target task conditioned on the proxy information, generated as described in the main text, and their estimate of $\left( \param, \tsparam \right)$ used standard Bayesian updating to condition on the source data.
        
        The r-weighted Bayesian learner estimated the model
        $$\sourcey_i \vert \sourcex_i, \param, \tsparam_{n+1} \sim \mathrm{Binomial}\left( \mathrm{sigmoid}\left( \param \sourcex_{i,(1:4)}^{\top} + \tsparam_{n+1} \right), N_i \right)$$
        The r-weighted learner's estimate of $\left( \param, \tsparam_{n+1} \right)$ used the following proxy-informed r-weighted likelihood of the 23 source data points $p^{\mathcal{R}}\left( \sourcedata, \auxInfo \vert \param, \tsparam_{n+1} \right)$:
        \begin{align}
            p^{\mathcal{R}}\left( \sourcedata, \auxInfo \vert \param, \tsparam_{n+1} \right) &= \robustL\left( \sourcedata, \param, \tsparam = \tsparam_{n+1} \right) ~ p\left( \auxInfo \vert \tsparam_{n+1} \right) \nonumber \\
            &= p\left( \auxInfo \vert \tsparam_{n+1} \right) \prod_{i=1}^{n} p\left( \sourcey_i \vert \sourcex_i, \param, \tsparam_i = \tsparam_{n+1} \right)^{\csim{i}{\tsparam_{n+1}}}. \nonumber
        \end{align}

        \paragraph{Proxy information.}
            To simulate proxy information, we sampled $\auxInfo \sim \mathcal{N}\left( \targettsparam, \sigma \right) + \mathbbm{1}_{noisy} \epsilon$ where $\mathbbm{1}_{noisy}$ indicates whether proxy contamination is present and $\epsilon \sim \mathcal{N}\left( 0, 3 \right)$ is the bias added to contaminate the proxy information.
            Since we do not know the true value $\targettsparam$, we approximated $\targettsparam$ by the mean of the corresponding fixed effect distribution estimated in a model that incorporated data from all 24 studies.
            
            When proxy information is \textit{weakly informative}, $\sigma = 3$ and $\mathbbm{1}_{noisy} = 0$.
            When proxy information is \textit{highly informative}, $\sigma = .1$ and $\mathbbm{1}_{noisy} = 0$.
            When proxy information is \textit{misleading}, $\sigma = 3$ and $\mathbbm{1}_{noisy} = 1$.
            The value of $\mathbbm{1}_{noisy}$ is unknown to the learner, who always models the proxy information as completely uncontaminated (i.e., as if $\mathbbm{1}_{noisy} = 0$).

    \subsection{Gaussian Process regression}\label{ap:gp}
        \begin{figure}[t!]
            \begin{subfigure}{.32\linewidth}
                \includegraphics[width=\linewidth]{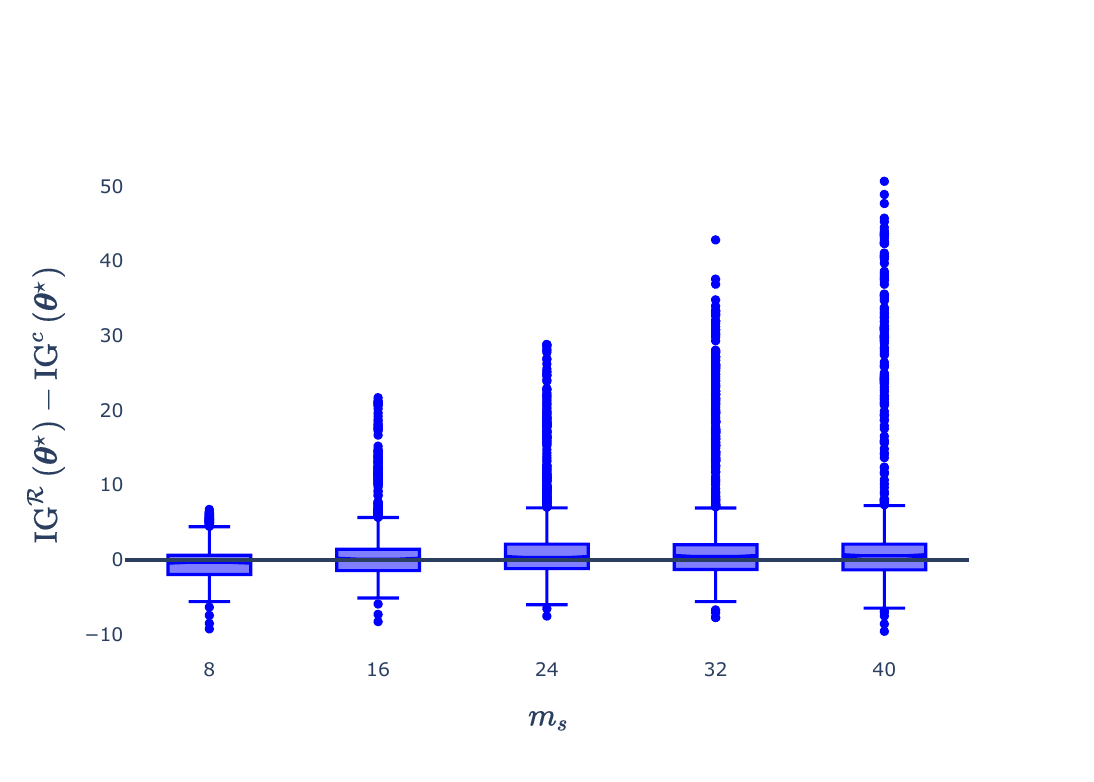}
                \caption{Amount of source data.
                    The advantage is more pronounced when more source data and proxy information is available.
                    See interpretation in the text.
                }
                \label{fig:gp-data-proxy}
            \end{subfigure}\hfill\begin{subfigure}{.32\linewidth}
                \includegraphics[width=\linewidth]{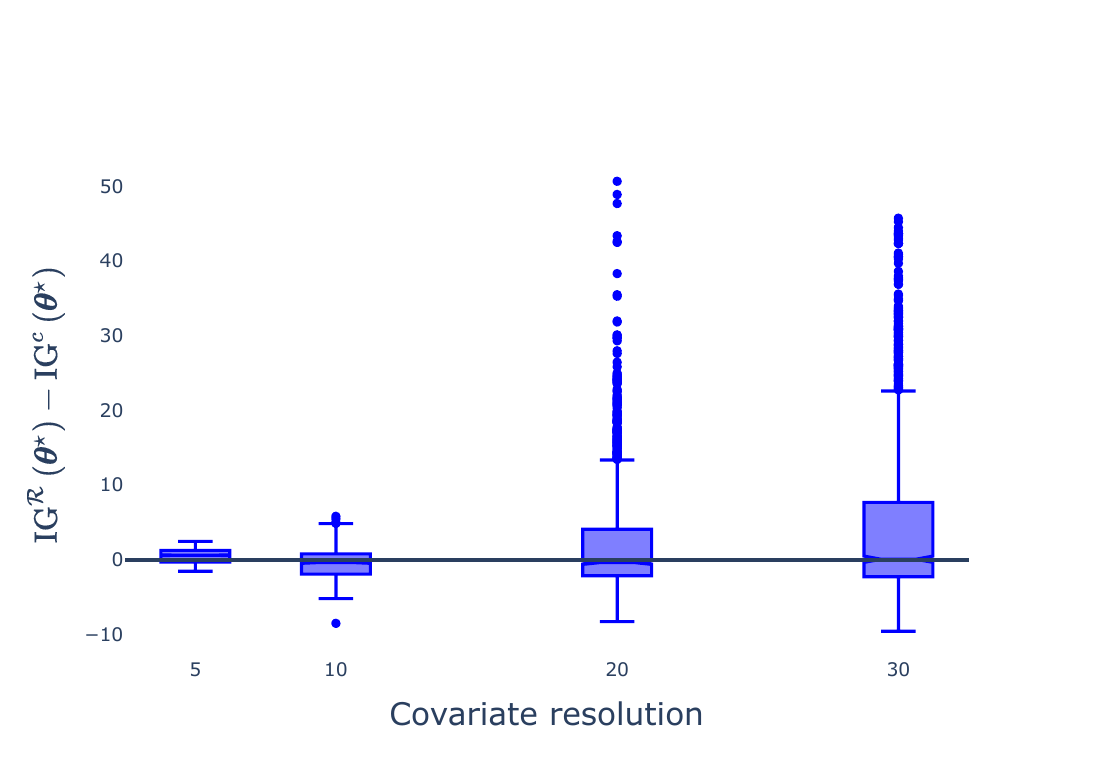}
                \caption{Covariate resolution.
                    The advantage is more pronounced for higher covariate resolutions.
                    See interpretation in the text. \\ ~
                }
                \label{fig:gp-res}
            \end{subfigure}\hfill\begin{subfigure}{.32\linewidth}
                \includegraphics[width=\linewidth]{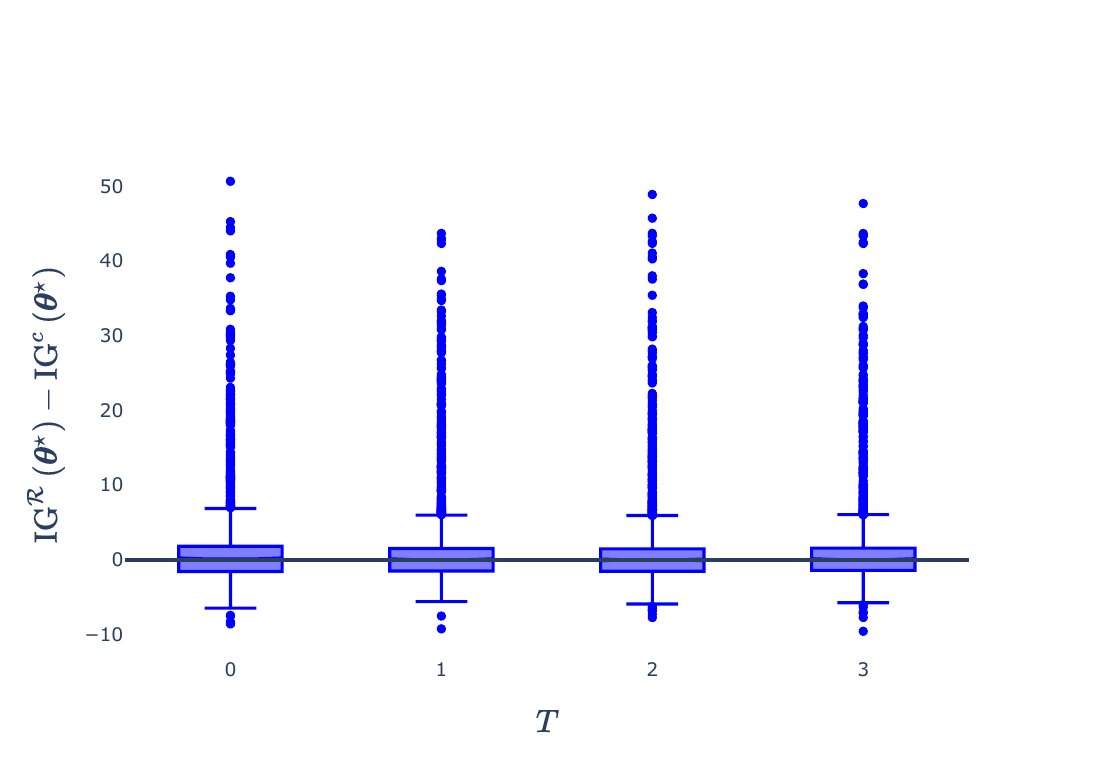}
                \caption{Number of iterations for refinement of the relevance function $T$. \\ ~ \\ ~}
                \label{fig:gp-T}
            \end{subfigure}
            \begin{subfigure}{.32\linewidth}
                \includegraphics[width=\linewidth]{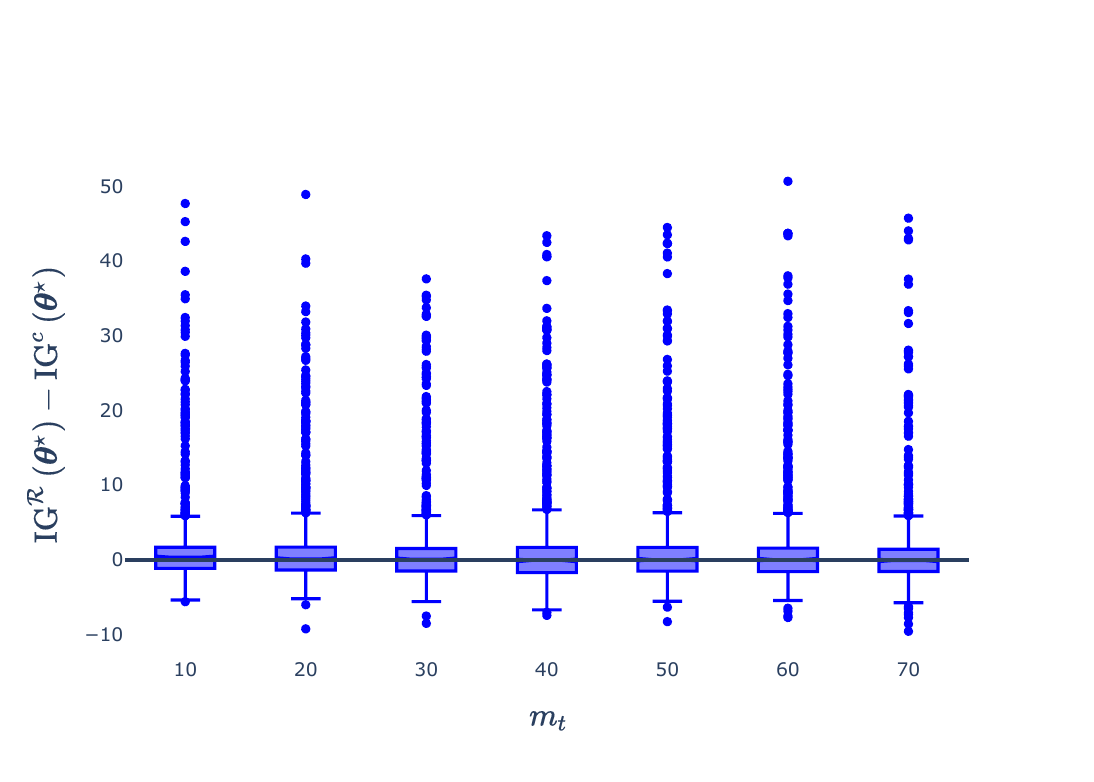}
                \caption{Amount of target information. \\ ~ \\ ~}
                \label{fig:gp-source-dist}
            \end{subfigure}\hfill\begin{subfigure}{.32\linewidth}
                \includegraphics[width=\linewidth]{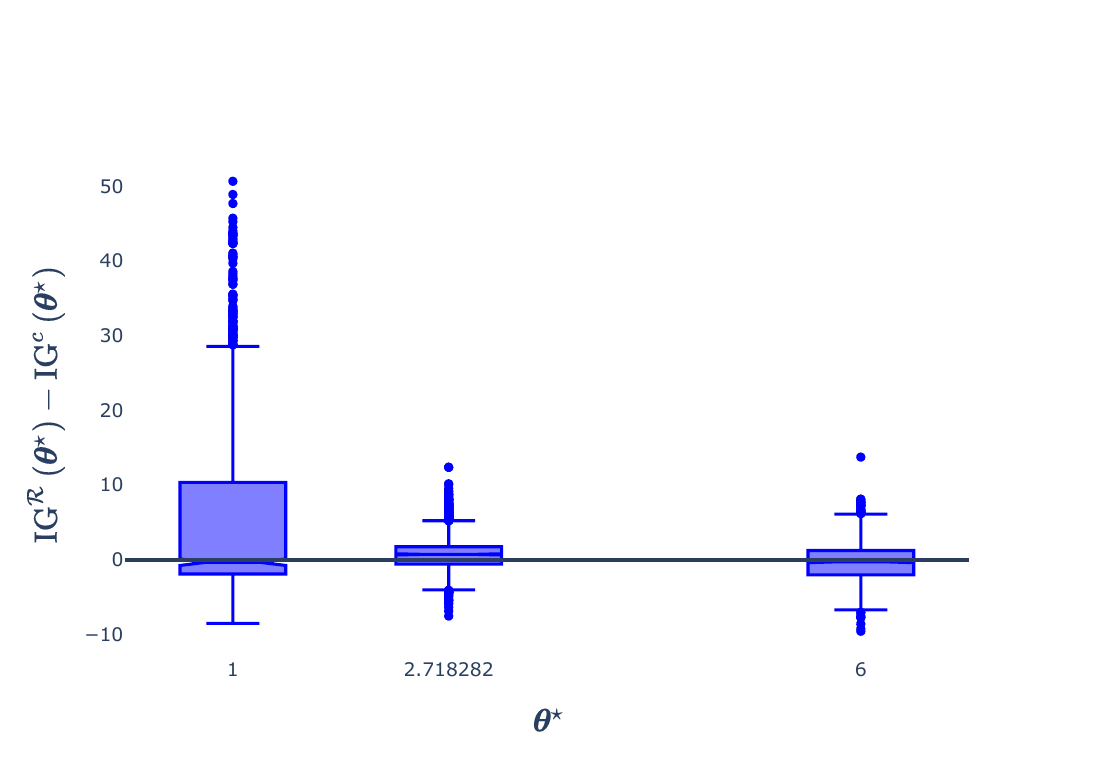}
                \caption{Value of $\targetconcept$.
                    The advantage is more pronounced for lower values of $\targetconcept$.
                    See interpretation in the text.
                }
                \label{fig:gp-target}
            \end{subfigure}\hfill\begin{subfigure}{.32\linewidth}
                \includegraphics[width=\linewidth]{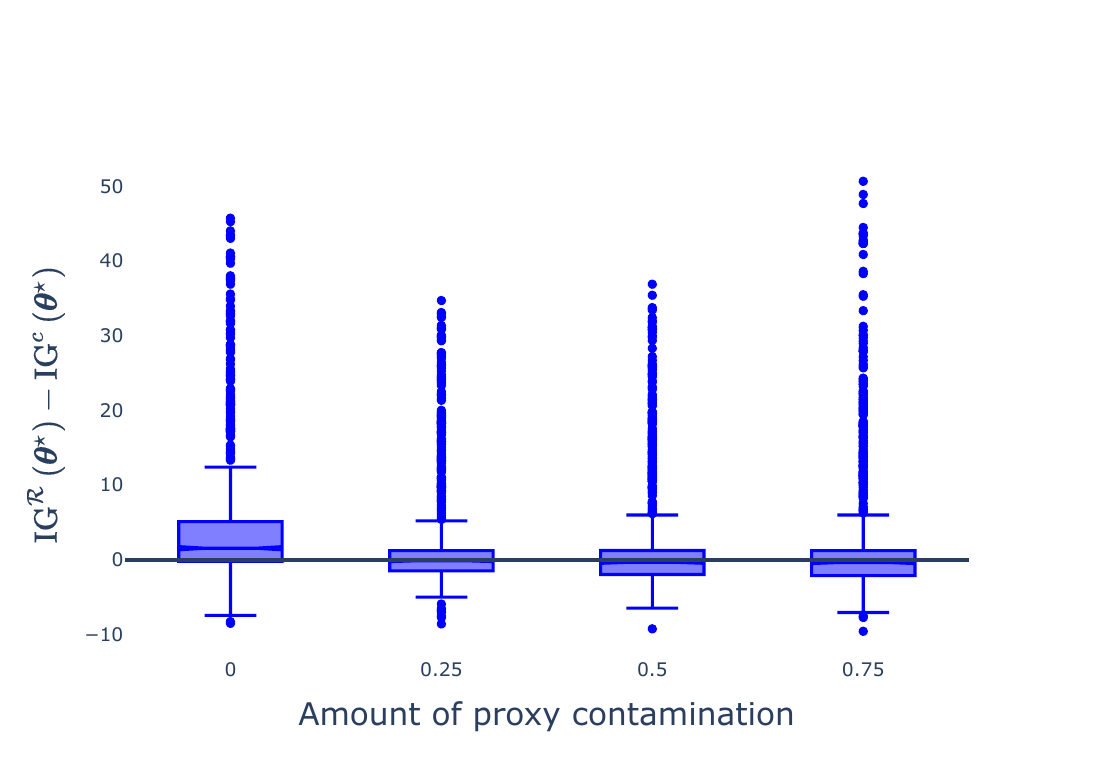}
                \caption{Amount of proxy contamination.
                See interpretation in the text of \Cref{sec:gp}. \\ ~ }
                \label{fig:gp-proxy-noise}
            \end{subfigure}
            \caption{Advantage of learning with an r-weighted likelihood in the GP regression setting as a function of the simulation parameter indicated in the subfigure caption.
                Each box in the plot shows the interquartile region (boxes) and outliers (points), across all values of all other simulation parameters, of the mean of $\rIG\left( \targetconcept \right) - \cIG\left( \targetconcept \right)$ across 50 simulations.
            }
            \label{fig:gp-supp}
        \end{figure}
        Each simulation was run on a single Nvidia A100 GPU.\footnote{The set of simulations run under 36 sets of simulation parameters (.5\% of all sets of simulation parameters) did not complete successfully.
            For an additional 22 sets of simulation parameters (.3\% of the total number of all sets of simulation parameters), all simulations encountered runtime errors.
        }
        The priors were $P_{\paramRV} = \mathrm{Lognormal}\left( 1, 1 \right)$ and $P_{\tsparamRV_i} = \mathrm{Gamma}\left( 3, .8 \right)$ for all $i \in 1:n+1$.
        
        For each simulation, we generated 80 trajectories drawn from a GP of the form given in the main text.
        From these 80 trajectories, trajectories $1:m_t$ were generated from the target task, where the \textit{amount of target information} $m_t$ was a variable simulation parameter (see below).
        Trajectories $(m_t+1):80$ were generated under a task parameter sampled at random from the learner's prior.
        Since these trajectories comprise most of the source data (see discussion of the effect of $m_t$ below), the learner's prior is in most cases relatively well-specified.
        In this sense, the results in \Cref{fig:gp-supp} are a somewhat conservative test of \algAbbrev{}.
        
        Trajectories $1:m_s$ were then used to create synthetic proxy information, while trajectories $(80-m_s):80$ were used for estimation of the target parameter (i.e., as source data), where the \textit{amount of source data} $m_s$ was a variable simulation parameter (see below).
        Proxy information was generated in the same way as for the linear regression example (see \Cref{ap:linreg}).
        
        The relevance function was computed using the same method described in \Cref{ap:linreg}, with the exception that the number of iterations $T$ used for refinement of the relevance function was a variable simulation parameter.

        After observing that results were affected by the value of some simulation parameters, we varied these parameters across simulations.
        We varied the following simulation parameters:
        \begin{itemize}
            \item \textbf{Amount of source data $m_s$:}
                This parameter, which took values in $\{ 8, 16, 24, 32, 40 \}$, controlled the number of trajectories in the source data.
                A distinct set of the same number of trajectories was used to create synthetic proxy information.
                Notice that the number of trajectories in the source data always equals the number of trajectories used to create proxy information, i.e., when more source data is available more proxy information is also available.
            
            \item \textbf{Covariate resolution:}
                Trajectories were evaluated on a grid of evenly-spaced values $\sourcex$ ranging from 0 to 1.
                This parameter, which took values in $\{ 5, 10, 20, 30 \}$, controlled the resolution and size of that grid.

            \item \textbf{Number of iterations for refinement of the relevance function $T$:}
                This parameter, which took values in $\{ 0, 1, 2, 3 \}$, controlled the number of iterations used for refinement of the relevance function.

            \item \textbf{Amount of target information $m_t$:}
                This parameter, which took values in $\{ 10, 20, 30, 40, 50, 60, 70 \}$, controlled the number of trajectories generated by the target task.
                Notice that observations from the target task are almost exclusively used to create synthetic proxy information (the exception is when $m_t > 80-m_s$, in which case the source data contains $m_t + m_s - 80$ trajectories from the target task).
                This reflects that the learner uses proxy information \textit{instead of} direct observations from the target task (i.e., instead of fine-tuning in the target task).
                The amount of target information $m_t$ to a certain extent admits a parallel interpretation as the number of trajectories a learner with both \textit{knowledge of the data sources} and \textit{the ability to fine-tune} (neither of which are available to the learners in our setting) would have to adapt to the target task, i.e., as the cost of operating in the setting of unknown data sources.

            \item \textbf{Value of $\targetconcept$:}
                We set $\targetconcept$ to either 1 (left tail of $P_{\paramRV}$), $e$ (mode of $P_{\paramRV}$), or 6 (right tail of $P_{\paramRV}$).

            \item \textbf{Amount of proxy contamination:}
                We generated and contaminated synthetic proxy information in the same way as described for the linear regression example in \Cref{sec:treatment-effect}.
                This parameter, which took values in $\{ 0, .25, .5, .75 \}$, controlled the fraction of proxy values which were contaminated.
        \end{itemize}

        \Cref{fig:gp-supp} shows how the relative performance of \algAbbrev{} depends on the value of each of the simulation parameters listed above.

        \Cref{fig:gp-data-proxy} shows that \algAbbrev{}'s advantage is more pronounced when more source data and proxy information are available.
        In other words, when the source data is sparse, the classic learner performs on par with the r-weighted learner.
        We speculate that this result reflects that in cases of source data sparsity both methods gain equally little information about $\targetconcept$.
        \Cref{fig:gp-res} shows a similar effect of the informativeness of the source data: \algAbbrev{}'s advantage is more pronounced for higher covariate resolutions.
        When the covariate resolution is low, the covariates are relatively far apart and so all observations will be relatively uncorrelated regardless of the value of the shared and task parameters.
        In these cases, observations provide little information about the smoothness of the underlying function.
        Like in cases of source data sparsity, we speculate that this result reflects that in cases of disparate observations both methods gain equally little information about $\targetconcept$.

        \Cref{fig:gp-target} shows that \algAbbrev{}'s advantage is more pronounced for smaller values of $\targetconcept$.
        This may be because of the effect of the value of $\targetconcept$ on the threat of negative transfer: 
        Values of $\tsparam^{\star}_i$ tend to be large (the distribution from which source task parameters are drawn is right-skewed), and the learner partially attributes the effect of a larger bandwidth in the task-specific component of the kernel to the shared component of the kernel.
        When the bandwidth of the shared component of the kernel is small, the result is negative transfer.
        In this sense, \Cref{fig:gp-target} corroborates the result shown in \Cref{fig:betabinom} that r-weighting is especially effective in the presence of the threat of negative transfer.

\bibliography{bibliography}
\end{document}